\documentclass[11pt]{article}
\usepackage[margin=1in]{geometry}

\usepackage[authoryear,round]{natbib}

\RequirePackage{amsthm,amsmath,amsfonts,amssymb}
\RequirePackage[authoryear]{natbib}%% uncomment this for author-year citations
\RequirePackage[colorlinks,citecolor=blue,urlcolor=blue,filecolor=red]{hyperref}
\RequirePackage{graphicx}
\usepackage{enumitem}
\usepackage{soul}
\usepackage{booktabs}
\usepackage{bm}
\usepackage{algorithm,algorithmic}
\usepackage{booktabs,multirow}
\usepackage{adjustbox}

\usepackage[dvipsnames]{xcolor}
\usepackage{booktabs}
\usepackage{array}

\usepackage{setspace}

\newtheorem{assumption}{Assumption}

\newtheorem{theorem}{Theorem}
\newtheorem{lemma}{Lemma}

\newtheorem{corollary}{Corollary}

\newtheorem{remark}{Remark}

\usepackage{bm}

\newcommand{\best}[1]{\ifmmode\bm{#1}\else\textbf{#1}\fi}
\newcommand{\second}[1]{\ifmmode\underline{#1}\else\underline{#1}\fi}
\newcolumntype{L}[1]{>{\raggedright\arraybackslash}p{#1}}

\usepackage{tcolorbox}
\usepackage[toc,page]{appendix}

\title{Kernelized Advantage Estimation: From Nonparametric Statistics to LLM Reasoning}
\author{
\begin{tabular*}{0.85\textwidth}{@{\extracolsep{\fill}}cccc@{}}
Shijin Gong$^{1,*}$ &
Kai Ye$^{2,*}$ &
Jin Zhu$^{3}$ &
Xinyu Zhang$^{4,1}$
\end{tabular*}
\\[0.6em]
\begin{tabular*}{0.45\textwidth}{@{\extracolsep{\fill}}cc@{}}
Hongyi Zhou$^{5,\dagger}$ &
Chengchun Shi$^{2}$
\end{tabular*}
}

\date{}

\begin{document}
\maketitle
\begingroup
\renewcommand{\thefootnote}{\fnsymbol{footnote}}
\footnotetext[1]{Equal contribution.}
\footnotetext[2]{Corresponding author.}
\endgroup

% numeric affiliation footnotes
\begingroup
\renewcommand{\thefootnote}{\arabic{footnote}}
\footnotetext[1]{School of Management, University of Science and Technology of China; Email: \texttt{shijin49@mail.ustc.edu.cn}.}
\footnotetext[2]{Department of Statistics, London School of Economics and Political Science; Emails: \texttt{K.Ye1@lse.ac.uk}, \texttt{c.shi7@lse.ac.uk}.}
\footnotetext[3]{School of Mathematics, University of Birmingham; Email: \texttt{j.zhu.7@bham.ac.uk}.}
\footnotetext[4]{Academy of Mathematics and Systems Science, Chinese Academy of Sciences; Email: \texttt{xinyu@amss.ac.cn}.}
\footnotetext[5]{Department of Mathematics, Tsinghua University; Email: \texttt{zhou-hy21@mails.tsinghua.edu.cn}.}
\endgroup

\onehalfspacing     
\begin{abstract}
Recent advances in large language models (LLMs) have increasingly relied on reinforcement learning (RL) to improve their reasoning capabilities. Three types of approaches have been widely adopted: %(i) %Proximal policy optimization and advantage actor-critic 
The first relies on a deep neural network to estimate the value function of the learning policy in order to reduce the variance of the policy gradient. However, estimating and maintaining such a value network incurs substantial computational and memory overhead. %(ii) Group relative policy optimization (GRPO) 
The second avoids training a value network by approximating the value function using sample averages. However, it samples a large number of reasoning traces per prompt for accurate value function approximation, making it computationally expensive. The third samples only a single reasoning trajectory per prompt, which reduces computational cost but suffers from poor sample efficiency.

This paper focuses on a practical, resource-constrained setting in which only a small number of reasoning traces can be sampled per prompt, while low-variance gradient estimation remains essential for high-quality policy learning. To address this challenge, we bring classical nonparametric statistical methods, which are both computationally and statistically efficient, to LLM reasoning. We employ kernel smoothing as a concrete example for value function estimation and the subsequent policy optimization. Numerical and theoretical results demonstrate that our proposal achieves accurate value and gradient estimation, leading to improved policy optimization.
\end{abstract}

\section{Introduction}

Large language models (LLMs) have achieved remarkable success across a wide range of tasks. From early developments in large-scale pretraining, which aims to model the distribution of human language for next-token prediction, to later improvements through task-specific post-training \citep{ouyang2022training}, the evolution of LLMs has progressively shifted from merely ``speaking well'' to ``problem solving''. The emergence of large reasoning models has further accelerated this trend, demonstrating human-level or superhuman performance in complex tasks such as mathematical problem solving and code generation \citep{jaech2024openai}. As detailed in Section \ref{sec:relatedwork}, progress in reasoning has evolved from early prompting-based approaches \citep{wei2022chain}, which explicitly encourage models to generate intermediate reasoning steps without additional training, to more recent large-scale reinforcement learning (RL)–based methods that directly retrain the model by exploring various reasoning trajectories to optimize its outcome \citep{lambert2024tulu}.

Such RL-based reasoning algorithms are mostly policy gradient algorithms \citep[Chapter 13]{sutton2018reinforcement}. These methods, with REINFORCE \citep{williams1992simple} as a prototype, are stochastic gradient ascent algorithms \citep{robbins1951stochastic} in nature that update model parameters by estimating the gradient of the expected outcome under the current policy. In practice, this gradient is approximated by sampling reasoning trajectories from the policy. When tackling complex problems, the model becomes highly uncertain about how to reason, leading to large variability across sampled trajectories. This variability, in turn, introduces substantial noise into the gradient estimation. Consequently, a central objective in policy gradient RL is to obtain sample efficient gradient estimates.

There are two dominant approaches to reducing the variance of gradient estimation. The first type of approaches, including proximal policy optimization \citep[PPO,][]{schulman2017proximal} and advantage actor-critic \citep[A2C,][]{mnih2016asynchronous}, reduces variance by modeling a value function  parameterized by a separate deep neural network in addition to the LLM itself, to serve as a baseline. While sample efficient, these methods require to train and store the deep value network, which incurs substantial computational cost and memory overhead for reasoning tasks. The second type of approaches, represented by group relative policy optimization \citep[GRPO,][]{shao2024deepseekmath} and its variants such as GSPO \citep{zheng2025group}, Dr.~GRPO \citep{liu2025understanding} and GPG \citep{chu2025gpg}, eliminates the value network completely. However, these approaches require generating a sufficiently large number of reasoning trajectories per prompt to accurately approximate the value function, which can be computationally expensive in practice. 

This paper considers a resource-constrained setting in which training a separate value network is computationally infeasible, rendering PPO- or A2C-type algorithms impractical. While GRPO-type algorithms remain applicable, only a limited number of reasoning trajectories can be sampled per prompt. As a result, the variance of the resulting gradient estimator can be substantial, which in turn lowers the quality of the learned policy \citep{greensmith2004variance}. Such settings are common in universities, small research labs, and public sectors with limited computational resources.

To address this challenge, we take a different perspective. Instead of relying on deep value networks or large-scale sampling, we draw inspiration from classical nonparametric statistical methods, which are doubly efficient, both computationally and statistically. More specifically, we leverage nonparametric statistics to enhance value function estimation, employ kernel smoothing \citep{nadaraya1964estimating} as a concrete proposal, and incorporate the resulting estimates into RL-based policy optimization.

\begin{figure}[t]
  \centering
  \includegraphics[width=0.8\linewidth]{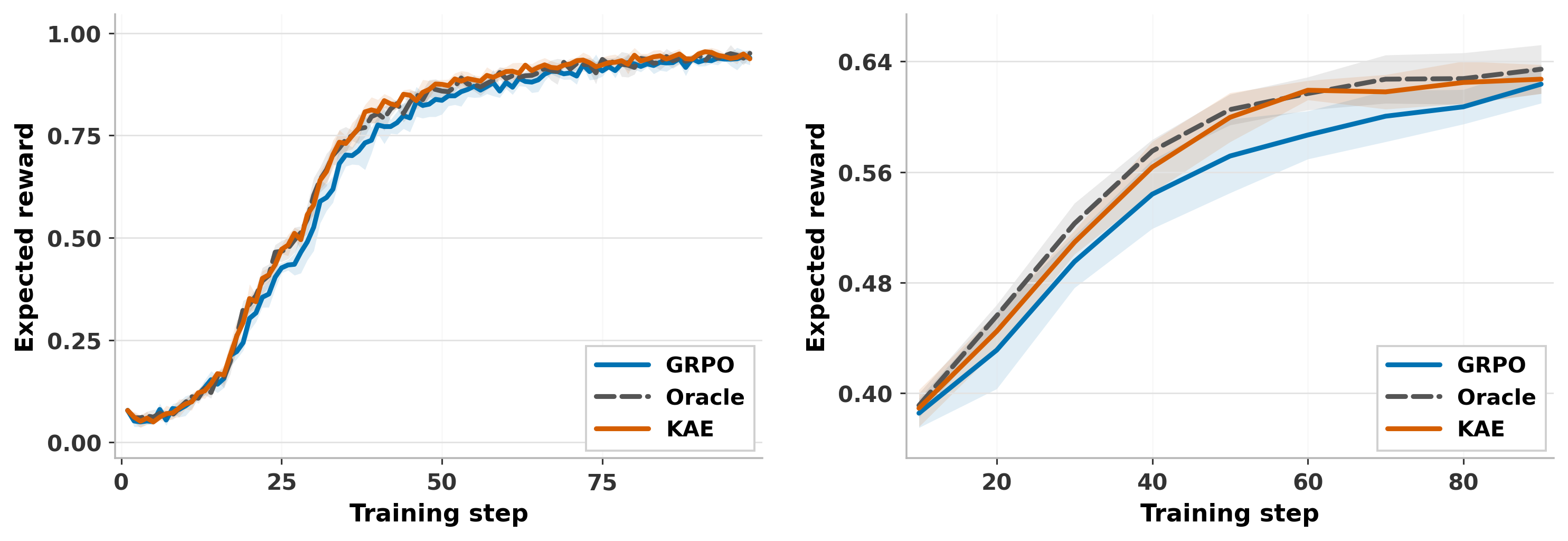}
  \caption{\small Expected rewards of one-shot GRPO \citep{wang2025reinforcement}, the oracle algorithm, and our method (denoted as KAE) on training (left) and testing (right) datasets in the one-shot regime where the training data consists of a single observation. One-shot GRPO applies the standard GRPO algorithm directly to this regime. 
  Shaded areas represent confidence intervals.}
  \label{fig:oneshot}
\end{figure}

Most notably, the proposed methodology enjoys an ``oracle'' property (see Figure~\ref{fig:oneshot} for an illustration): its performance on test data is comparable to that of an oracle algorithm that knows the oracle value function in advance and uses it for gradient estimation. This is particularly remarkable because, under limited computational budgets, we cannot afford to train complex neural networks for value estimation, nor can we sample a large number of trajectories to approximate it. Nevertheless, our method achieves performance comparable to that of an oracle with full access to the true value function. %and  consistently outperforms GRPO in resource-constrained settings.

We validate our proposal both theoretically and empirically. Theoretically, we establish the oracle property of our method and demonstrate its advantages over GRPO- and REINFORCE-type algorithms along three dimensions: the mean squared error (MSE) of the value function estimator (Theorem~\ref{thm:value} \& Corollary~\ref{coro:value}), the MSE of the gradient estimator (Theorem~\ref{thm:grad} \& Corollary \ref{coro:grad}), and the performance of the learned policy (Theorem~\ref{thm:policy} \& Corollary \ref{coro:policy}). Empirically, we conduct extensive experiments to validate these advantages over both GRPO- and REINFORCE-type algorithms in value estimation (Table~\ref{tab:baseline_mse_cross}), gradient estimation (Table~\ref{tab:gradvar_stage_summary}) and  policy optimization (Tables~\ref{tab:qwen7b_math35} and \ref{tab:qwen7b_dapo17k}, Figure~\ref{fig:qwen15b_n1_val_curves}). Additionally, ablation studies confirm that these gains stem from improved value function estimation (Figures~\ref{fig:qwen15b_n1_val_curves} and \ref{fig:ablation_math_curves}). Finally, sensitivity analysis shows that our value estimator's MSE is not overly sensitive to the choice of kernel bandwidth and kernel function  (Figure~\ref{fig:mse_kernel_sensitivity}).

\subsection{Related Work}\label{sec:relatedwork}
Our work sits at the intersection of artificial intelligence (AI) and statistics. On the AI side, it is closely related to the rapidly growing literature on LLM reasoning. On the statistics side, it connects to  classical work on nonparametric estimation and modern work on RL.

\textbf{LLM reasoning}. LLM reasoning methods can be broadly grouped into three categories: (1) prompting-based approaches; (2) inference-time\footnote{Here, ``inference-time'' refers to the stage at which a trained model is applied without retraining. This differs from the notion of statistical inference, which studies e.g., uncertainty quantification for an estimator.} approaches that enhance reasoning through search; and (3) training-time approaches via alignment or RL.

Early work falls into the first two categories. In particular, prompting-based methods such as chain-of-thought  prompting guide LLMs to produce step-by-step reasoning in a manner similar to humans \citep{wei2022chain}. This is often achieved by simple, magical prompts such as ``Let us think step by step.'' Under such instructions, the model generates a chain of thought that decomposes a complex problem into a sequence of intermediate subproblems and produces the reasoning process before arriving at a final answer. Instead of generating a single reasoning trajectory, inference-time approaches explore multiple candidate reasoning paths and select the best one as the final output \citep{zhang2026dagmath,zhu2026align}. 

More recent work has increasingly focused on training-time approaches, such as reinforcement learning from verifiable rewards (RLVR), which fine-tune model parameters using RL to directly enhance LLM reasoning capabilities \citep[see, e.g.,][]{lambert2024tulu,hu2025reinforce,dai2025cde,huang2026learning}. These approaches are most closely related to our proposal. They rely on policy gradient  algorithms such as REINFORCE, A2C and PPO, which estimate the gradient of the expected return of the policy using sampled trajectories and update the model parameters via stochastic gradient ascent. However, REINFORCE is well known to suffer from high variance in its gradient estimates. PPO and A2C mitigate this limitation by using the advantage function, defined as the difference between the return and the value function of the policy, in place of the return when constructing the gradient estimator. These methods estimate the value function via a deep neural network. However, maintaining and updating such a value network can be computationally intensive. 

A major breakthrough in this line of work is GRPO, the post-training algorithm underlying large reasoning models such as DeepSeekMath \citep{shao2024deepseekmath} and DeepSeek-R1 \citep{guo2025deepseek}, as well as a number of open-source LLMs \citep{liu2025fin,yang2025qwen3}. Its main idea is closely related to that of \citet{kool2019buy}: for each prompt, multiple reasoning trajectories are sampled, leading to multiple rewards, which are averaged to approximate the value function for variance reduction. In contrast to PPO and A2C, GRPO eliminates the need for learning a separate value network, which facilitates the computation while maintaining statistically efficient policy optimization.

GRPO has inspired a large number of follow-up methods \citep[e.g.,][]{hao2025on,li2025repo,lin2025cppo,xiong2025a,yan2025learning,zhao2025geometric,zheng2025parallel,li2026disco,li2026bicc}. 
Our work is most closely related to those that focus on improving the statistical or sample efficiency of value and advantage estimators. Among those available, \citet{zeng2025shrinking} and \citet{han2026ebpo} propose shrinkage estimators that replace GRPO's per-prompt reward average with a combination of per-prompt and cross-prompt averages. The intuition behind these methods is closely related to the James--Stein estimator \citep{james1961estimation} in classical statistics. Unlike these approaches, which borrow information across prompts within a single training iteration to improve estimation accuracy, our proposal borrows information across training iterations for the same prompt (see Section \ref{sec:alg}). \citet{wang2025krpo} and \citet{xu2025single} adopt similar ideas and employ Kalman filtering or Bayesian methods to smooth rewards across training iterations. Our proposal differs from these methods in three respects. Methodologically, we employ kernel smoothing rather than Kalman filtering, and we additionally design the prompt sampling schedule to further improve value and advantage estimation accuracy (see Section \ref{sec:experiment}). Theoretically, we establish learning guarantees for the value estimator, the gradient estimator, and the resulting policy. These theoretical results are largely absent from the aforementioned work.

\textbf{Nonparametric statistics and RL}. Nonparametric statistical methods estimate regression functions without imposing restrictive parametric assumptions. Classical approaches include kernel smoothing, local polynomial regression \citep{stone1977consistent}, and sieve estimators \citep{grenander1981abstract}. Statistically, these estimators can achieve the optimal convergence rates established by \citet{stone1982optimal}; see, e.g., \citet{fan1997local,huang1998projection,chen2007large}. Computationally, these algorithms are much more efficient to implement compared to deep neural networks. This makes them well suited for value function estimation in LLM reasoning under resource constraints. In this work, we employ kernel smoothing  as a concrete example, although the framework naturally extends to other nonparametric estimators.

More recently, there has been growing interest in developing RL algorithms in the statistics literature. These methods can be broadly categorized into three classes: (i) approaches designed for learning optimal dynamic treatment regimes without imposing Markov assumptions on the data generating process \citep[see, e.g.,][for reviews]{chakraborty2013statistical,laber2014dynamic,kosorok2019precision,tsiatis2019dynamic,ge2025reviewcausaldecisionmaking,gazi2026statisticalreinforcementlearningreal}; (ii) approaches developed for Markov decision processes (MDPs) under the Markov assumption \citep[e.g.,][]{ertefaie2018constructing,luckett2020estimating,liao2022batch,wang2023projected,chen2024reinforcement,li2024settling,shi2024value,zhou2024estimating,zhou2024federated,bian2025off,chai2025deep,jin2025policy,li2025reinforcement,miao2025reinforcement,liu2025online,zhong2025risksensitive}; and (iii) approaches tailored to RL from human or AI feedback \citep[e.g.,][]{lee2024lowrankcontextualreinforcementlearning,liu2025statistical,liu2025uncertaintyquantificationlargelanguage,lu2025contextual,xiao2025algorithmic,cho2026privacy,liu2026reinforcement,xia2026statistical}. 
Our work is related to those methods based on A-learning, which focus on estimating advantage (contrast, or blip) functions for optimal policy learning \citep{murphy2003optimal,robins2004optimal,lu2013variable,shi2018high,liang2023relative,shi2024statistically}. However, we study a fundamentally different application in LLM reasoning, leading to substantially different methodologies.

\section{Preliminaries: RLVR for LLM Reasoning}\label{sec:prelim}
We adopt a contextual bandit framework \citep[see, e.g.,][]{lai1985asymptotically} to formulate the RLVR problem in LLM reasoning. At each time step, the LLM receives a user query, referred to as a \textit{prompt} $X$. To address this query, the LLM generates a reasoning trajectory along with a final answer. Together, the reasoning trajectory and the answer form a \textit{completion} $Y$. Both $X$ and $Y$ are represented as sequences of tokens via tokenization, where each token (word, subword, or punctuation) is mapped to a unique integer according to a vocabulary that collects all possible tokens. The pair $(X, Y)$ is then evaluated by a verifiable reward function $r$, yielding a scalar \textit{reward} $Z = r(X, Y)$. In mathematics, each problem typically has a unique correct solution, and the reward can be defined as $1$ if the generated answer matches the ground truth and $0$ otherwise. Similarly, in coding tasks, LLM-generated programs can be executed to verify whether they pass the corresponding test cases.

Leading LLMs, such as GPT, Gemini, and Claude, are all autoregressive models. From an RL perspective, they can be viewed as policy networks, denoted by $\pi_{\theta}$, parameterized via the Transformer architecture. Specifically, given an input token sequence $x$ of arbitrary length, $\pi_{\theta}(\bullet |x)$ defines a probability mass function over the vocabulary, representing the distribution of the next token. A completion is generated autoregressively: the model first takes $X$ as input and produces the first token $Y_1$, then conditions on the concatenated sequence $(X, Y_1)$ to generate $Y_2$, and proceeds iteratively until an end-of-sequence token $Y_T$ is produced. This yields the full completion $Y = (Y_1, \ldots, Y_T)^\top$. 
The objective of RLVR is to identify the optimal parameter $\theta^*$ that maximizes the expected reward:
{
\setlength{\abovedisplayskip}{8pt}
\setlength{\belowdisplayskip}{8pt}
\setlength{\abovedisplayshortskip}{3pt}
\setlength{\belowdisplayshortskip}{3pt}
\begin{equation*}
    \theta^* = \arg\max_{\theta \in \Theta} J(\theta)\,\, \hbox{where}\,\,
    J(\theta) := \mathbb{E}^{\pi_{\theta}}(Z),
\end{equation*}}
where the expectation $\mathbb{E}^{\pi_{\theta}}$ is taken with respect to the distribution over completions induced by the policy $\pi_{\theta}$.

As discussed in Section~\ref{sec:relatedwork}, existing RLVR algorithms are policy gradient methods. They are motivated by the observation that the gradient of $J(\theta)$ can be expressed as
{
\setlength{\abovedisplayskip}{8pt}
\setlength{\belowdisplayskip}{8pt}
\setlength{\abovedisplayshortskip}{3pt}
\setlength{\belowdisplayshortskip}{3pt}
\begin{equation*}
    \nabla_{\theta} J(\theta) = \mathbb{E}^{\pi_{\theta}} \big[ Z \nabla_{\theta} \log \pi_{\theta}(Y |X) \big],
\end{equation*}}
that is, as the expectation of the product of the reward $Z$ and the score function $\nabla_{\theta} \log \pi_{\theta}(Y |X)$ (referred to as the policy score), where $\log \pi_{\theta}(Y |X)=\sum_{t=1}^T \log \pi_{\theta}(Y_t |X, Y_1, \cdots, Y_{t-1})$. This representation gives birth to the following stochastic gradient ascent (SGA) algorithm: at each iteration, a minibatch of prompts $\{X^{(b)}\}_{b=1}^B$ is sampled, a completion $Y^{(b)}$ is generated for each prompt, rewards $\{Z^{(b)}\}_{b=1}^B$ are obtained via the verifiable reward function, and the corresponding policy scores are computed. The gradient is then estimated by averaging the products of rewards and policy scores across the sampled batch, denoted by $\widehat{g}(\theta)$, and the policy parameter is updated via $\theta \leftarrow \theta +\eta \widehat{g}(\theta)$ for some learning rate parameter $\eta$. 

When confined to the classical MDP setting, such an algorithm is known as REINFORCE. However, REINFORCE’s gradient estimators are well known to suffer from high variance, which lowers the quality of the learned policy. Three major approaches have been developed in the literature to address this limitation:
\begin{enumerate}[leftmargin=*]
    \item The first approach is A2C, which introduces a critic function $C(X)$ to serve as a baseline and replaces the reward $Z$ with an advantage function $A = Z - C(X)$ in constructing the policy gradient estimator $\widehat{g}(\theta)$. Its main idea is that $\nabla_{\theta} \log \pi_{\theta}(Y|X)$ is a score function, and thus multiplying it by any $C(X)$ yields a vector with zero expectation. Consequently, subtracting $C(X)$ does not bias the gradient estimator. However, by choosing an appropriate baseline, its variance can be substantially reduced. Under an uncorrelatedness condition introduced in Section \ref{sec:theory}, the optimal baseline that minimizes the variance is the value function $V^{\pi_{\theta}}(X)$ \citep{greensmith2004variance}. This motivates A2C, which maintains a separate value network $\widehat{V}$ to approximate the value function and uses the plug-in estimator $Z - \widehat{V}(X)$ as the advantage for constructing the policy gradient.
    \item The second approach is REINFORCE++ \citep{hu2025reinforce}, which replaces the critic with a simple baseline, given by the average reward across all prompts at the current training step, i.e., $\bar{Z}=B^{-1}\sum_{b=1}^B Z^{(b)}$, to eliminate the need for the value network and facilitate the computation. 
    \item The third approach is GRPO-type algorithms, which sample a group of completions $\{Y^{(b,g)}\}_{g=1}^G$ for each prompt $X^{(b)}$ to eliminate the need for a value network. Specifically, for each completion $Y^{(b,g)}$, the corresponding reward is computed as $Z^{(b,g)} = r(X^{(b)}, Y^{(b,g)})$, and the within-group average $\bar{Z}^{(b)} = G^{-1} \sum_{g=1}^G Z^{(b,g)}$ is used as a proxy for the value function. This leads to an advantage function of the form $A^{(b,g)} = Z^{(b,g)} - \bar{Z}^{(b)}$, yielding a variance-reduced policy gradient estimator without requiring a learned value network. 
\end{enumerate}
\begin{figure}[t]
    \centering
    \includegraphics[width=0.8\textwidth]{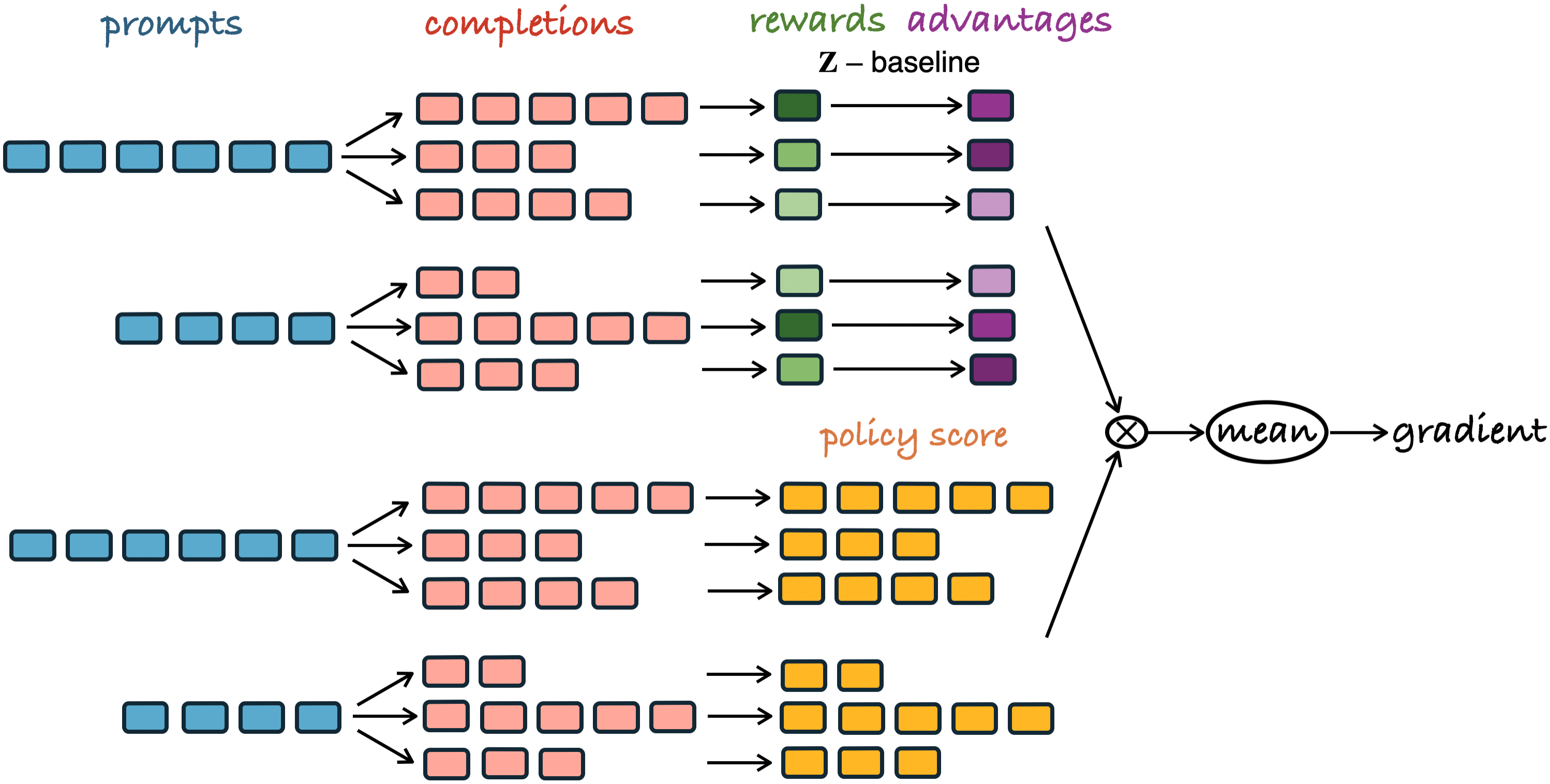}
    \caption{\small Illustrations of a generic algorithm that unifies A2C, REINFORCE- and GRPO-type algorithms. }
    \label{fig:metaalgorithm}
\end{figure} 
Figure \ref{fig:metaalgorithm} unifies the three approaches by introducing a generic baseline for advantage function estimation and illustrates the resulting gradient estimator. When a single completion is sampled per prompt and the baseline is set to the empirical average of rewards across all prompts, the method reduces to REINFORCE++. When the baseline is given by a learned value network, it recovers A2C. Finally, when a group of completions are sampled and the baseline is set to be the empirical group mean, it corresponds to GRPO-style algorithms.

Despite their statistical efficiency, both A2C and GRPO incur substantial computational costs. A2C is computationally intensive due to the requirement to maintain and update a separate value network.
In contrast, GRPO avoids a learned critic but relies on sampling multiple completions per prompt to obtain a sufficiently accurate gradient estimator. In practice, this group size can be large (e.g., $G=64$ in the DeepSeekMath paper). REINFORCE++, on the other hand, reduces computational cost by sampling only a single completion per prompt. However, its baseline aggregates rewards across all prompts and is a biased estimator of the value function for a specific prompt. Consequently, it lowers the statistical efficiency. 
In summary, achieving both statistical and computational efficiency in gradient estimation remains challenging. This is precisely where nonparametric statistical methods can play a central role in achieving both types of efficiency, and it motivates the proposed methodology. 

\section{Kernelized Advantage Estimation}
\label{sec:alg}
We first present the motivation of the proposed algorithm. We focus on a resource-constrained setting where the number of completions that can be sampled per prompt is limited. As a result, GRPO's value estimator, constructed via averaging rewards over multiple completions, suffers from large variance. To address this limitation, we propose to borrow information across different training iterations to improve the estimation accuracy of the value function, and consequently, the advantage function.

The motivation for our algorithm is best illustrated in the one-shot regime, where the training data consists of a single prompt $x$. Surprisingly, recent work suggests that LLMs can still generalize well even in such regimes \citep{wang2025reinforcement}. One possible explanation is that, although there is only one data example, it consists of many tokens, so the effective sample size is not necessarily one. We do not pursue this phenomenon further here. Instead, we use this setting to illustrate our main idea. 

In the one-shot regime, because the same prompt $x$ appears repeatedly over training, rewards collected at previous iterations also contain information about the current value function. Therefore, rather than estimating $V^{\pi_{\theta_i}}(x)$ solely from the current batch, we may borrow information from past rewards. Of course, historical rewards cannot be pooled naively. Since the policy parameter $\theta_i$ evolves over training, its value function $V^{\pi_{\theta_i}}(x)$ is generally different from $V^{\pi_{\theta_j}}(x)$ for $j < i$. Thus, older rewards should contribute less to the current estimate than more recent rewards. 

This leads naturally to a one-dimensional nonparametric regression problem: We treat the training iteration index $i$ as the predictor and the corresponding value function $V^{\pi_{\theta_i}}(x)$ as the target. The goal is to estimate $V^{\pi_{\theta_i}}(x)$ by combining rewards observed at both the current iteration and previous iterations. In this way, we connect value and advantage estimation in LLM reasoning to classical nonparametric regression, allowing us to leverage classical nonparametric statistical methods to borrow information over time. 

In this paper, we adopt kernel smoothing as a concrete nonparametric method to implement this idea. For simplicity, suppose that at each training iteration $i$ we sample a single completion and let $Z_i$ denote its reward. Let $K(\cdot)$ be a kernel function that satisfies $\int_0^{\infty} K(u)du=1$. We estimate the current value function by a kernel-weighted average of these rewards,
{
\setlength{\abovedisplayskip}{8pt}
\setlength{\belowdisplayskip}{8pt}
\setlength{\abovedisplayshortskip}{3pt}
\setlength{\belowdisplayshortskip}{3pt}
\begin{eqnarray*}
\widehat V_i(x)
=
\frac{1}{ih}\sum_{j=0}^{i-1} K\!\left(\frac{i-j}{ih}\right) Z_j,
\end{eqnarray*}}
where $h > 0$ is a bandwidth parameter. 
Due to the use of kernel smoothing for value estimation, and the resulting construction of the advantage estimator, we refer to our method as \emph{kernelized advantage estimation} (KAE).
\begin{algorithm}[!t]
\caption{Kernelized Advantage Estimation (KAE).}\label{alg1}
% \singlespacing
\begin{algorithmic}[1]
    \STATE \textbf{Input:} Prompt set $\mathcal{X}$, initial parameter $\theta_0 \in \Theta$, learning rates $\{\eta_i\}_{i=0}^{n-1}$, batch size $B$, per-prompt group size $G$, kernel function $K(\cdot)$, bandwidth $h$, and sets of historical rewards $\mathcal{H}_i(x)=\emptyset$ for any $i<n$ and $x\in \mathcal{X}$.
    \FOR{$i = 0,1,2,\dots,n-1$}
        \STATE Sample a minibatch of prompts $\{X^{(b)}\}_{b=1}^B$ from $\mathcal{X}$.
        \FOR{$b = 1,\dots,B$}
            \STATE Generate a group of completions $\{Y^{(b,g)}\}_{g=1}^G \sim \pi_{\theta_i}(\bullet | X^{(b)})$.
            \STATE Compute rewards $Z^{(b,g)} = r(X^{(b)}, Y^{(b,g)})$ for $g=1,\dots,G$.
\IF{$\mathcal{H}_i(X^{(b)})=\emptyset$} 
\STATE Set $\widehat V_i^{(g)}(X^{(b)})=(G-1)^{-1}\sum_{k\neq g} Z^{(b,k)}$ for $g=1,\cdots,G$.
            \ELSE
            \STATE \[
            \widehat V_i^{(g)}(X^{(b)})=
             \frac{1}{M_i(X^{(b)})}\Big[\sum_{(I_j,Z_{j}) \in \mathcal{H}_i(X^{(b)})} K\Big(\frac{i-I_j}{ih}\Big) Z_{j}+\sum_{k\neq g} K(0) Z^{(b,k)}\Big],
            \]
            where $M_i(X^{(b)})$ denotes the normalizing constant $h|\mathcal{H}_i(X^{(b)})|+(G-1)K(0)$. 
            \ENDIF
            \STATE For each $g=1,\dots,G$, define the advantage estimate $A^{(b,g)} = Z^{(b,g)} - \widehat V_i^{(g)}(X^{(b)})$.
        \ENDFOR
        \STATE Estimate the policy gradient:
        \[
        \widehat g_{\textrm{KAE}}(\theta_i)
        =
        \frac{1}{BG}
        \sum_{b=1}^B \sum_{g=1}^G
        A^{(b,g)} \nabla_{\theta}\log \pi_{\theta_i}(Y^{(b,g)}|X^{(b)}).
        \]
        \STATE Update the policy parameter $\theta_{i+1} \gets \theta_i + \eta_i \widehat g_{\textrm{KAE}}(\theta_i)$.
        \STATE Update the sets of historical rewards:
        \[
        \mathcal{H}_{i+1}(X^{(b)}) \gets \mathcal{H}_i(X^{(b)}) \cup \{(i, Z^{(b,g)}): g=1,\dots,G\},
        \qquad b=1,\dots,B.
        \]
    \ENDFOR
    \STATE \textbf{Output:} $\pi_{\theta_n}$.
\end{algorithmic}
\end{algorithm}
In more general regimes, prompts used in LLM post-training are often drawn from a fixed set (e.g., a collection of mathematical questions), denoted by $\mathcal{X}=\{x_1,x_2,\cdots,x_m\}$ with $m$ being the total number of prompts. Consequently, the same prompts may be encountered multiple times during training, and past rewards remain informative when estimating the value function under the current prompt. We summarize the complete procedure in Algorithm \ref{alg1} and highlight some main steps below.

Consistent with Section~\ref{sec:prelim}, we sample a minibatch of $B$ prompts $\{X^{(b)}\}_{b=1}^B$ at each training iteration (Line 3), generate $G$ completions $\{Y^{(b,g)}\}_{g=1}^G$ for each prompt (Line 5), and obtain their rewards $\{Z^{(b,g)}\}_{g=1}^G$. On Lines 7 and 8, the value function is estimated in a leave-one-out manner similar to \citet{kool2019buy}. Specifically, to construct the advantage estimator $A^{(b,g)}$ for the $g$th completion, we use all other rewards in the current group (excluding the $g$th), together with historical rewards associated with the same prompt. Although the $g$th reward could also be incorporated, doing so would only change the advantage estimator by a multiplicative constant asymptotically. Given these advantage estimates, the policy gradient is computed on Line 12 by averaging the product of the advantage and the policy score over all prompt–completion pairs. Finally, on Line 13, the model parameters are updated via stochastic gradient ascent using the learning rate $\eta_i$.

%{\color{red} [whether to discuss KAE can be applied to not only GRPO but many other methods here? Or in the Discussion section]}

\section{Theoretical Results}\label{sec:theory}
In this section, we establish the statistical properties of KAE. We begin with a high-level summary of our theoretical findings. As mentioned earlier, our theories are organized along three dimensions: (i) value estimation; (ii) gradient estimation and (iii) policy optimization. Their connections highlight the motivation behind KAE.

\textbf{Summary}. KAE is designed to improve value estimation. We first establish an MSE bound for the proposed value estimator (Theorem~\ref{thm:value}). %and show that, under an appropriate choice of kernel bandwidth, it achieves Stone’s optimal rate of convergence (Corollary~\ref{coro:value}). 
We next show that improvements in value estimation directly translate into improved gradient evaluation: the MSE of our gradient estimator differs from that of a corresponding \emph{oracle} algorithm, only by an extra term that is proportional to the MSE of the value estimator asymptotically (Theorem~\ref{thm:grad}). Here, the oracle algorithm is identical to Algorithm~\ref{alg1}, except that it uses the true value function to construct the advantage and the resulting policy gradient. Finally, we show that improved gradient estimation further translates into better policy optimization. In particular, the upper bound on the suboptimality gap of the learned policy, defined as the difference in expected value between the optimal policy and the learned policy, depends directly on the MSE of the gradient estimator (Theorem~\ref{thm:policy}).
Taken together, these results establish the oracle property of KAE and imply its advantages over REINFORCE- and GRPO-type algorithms along all three dimensions (Corollaries~\ref{coro:value} -- \ref{coro:policy}).

Next, we impose the following technical conditions. 
\begin{assumption}[I.i.d. sampled prompts]\label{assump:iidsampledprompt}
The minibatches of prompts $\{X^{(b)}\}_{b=1}^B$ are sampled i.i.d. across training iterations. At each iteration, a minibatch is drawn uniformly at random from $\mathcal{X}$ without replacement, and independently of previously sampled minibatches.
\end{assumption}

\begin{assumption}[Boundedness]\label{assump:boundedreward}
    $Z$ and $\nabla_{\theta}\log \pi_{\theta}(Y|X)$ are almost surely bounded. 
\end{assumption}

\begin{assumption}[Kernel function]\label{assump:kernelfun}
    $K$ is supported on $[0,1]$, bounded, Lipschitz continuous and satisfies $\int_{0}^1 K(u)du=1$. %and $\int_{0}^1 u |K(u)|du<\infty$.
\end{assumption}

\begin{assumption}[Smoothness]
\label{assump:psmooth}
For any $x\in\mathcal X$, $V^{\pi_{\theta}}(x)$ is
Lipschitz continuous with respect to $\theta$, i.e., there exists some constant $L$ such that $|V^{\pi_{\theta_1}}(x)-V^{\pi_{\theta_2}}(x)|\le L\|\theta_1-\theta_2\|$ for all $\theta_1,\theta_2$. Additionally, $\nabla_{\theta} J(\theta)$ is Lipschitz continuous with respect to $\theta$ as well.
\end{assumption}
\begin{assumption}[Learning rate]\label{assump:learningrate}
    The sequence of learning rates $\{\eta_i\}_i$ follows a $1/i$ schedule,  where $\eta_i = \beta/i$ for any $i\ge 1$ and some $\beta>0$.
\end{assumption}

The i.i.d. assumption in Assumption~\ref{assump:iidsampledprompt} is imposed for simplicity. In practice, prompts may be sampled dependently across iterations, as in our implementation (see Section~\ref{sec:experiment}). Assumption~\ref{assump:boundedreward} is mild. In LLM reasoning, rewards are bounded, for example, as a binary indicator of the correctness of the model’s output, potentially with an additional format reward that encourages the output to follow a prescribed form. Assumption \ref{assump:kernelfun} is standard in the kernel smoothing literature. 
Assumption~\ref{assump:psmooth} is also mild. The first part only requires the value function to be Lipschitz continuous in the policy parameter $\theta$, rather than differentiable. This accommodates policy networks with ReLU activations, which are Lipschitz continuous but not differentiable everywhere. The second part imposes a stronger smoothness condition on the objective $J(\theta)$. Since $J(\theta)$ is obtained by averaging $V^{\pi_\theta}(x)$ over $x$, such averaging may smooth out pointwise nonsmoothness of the value function. As a simple illustration, the function $\max(X-\theta,0)$ is nonsmooth in $\theta$ pointwise, but when $X\sim N(0,1)$, its expectation $\mathbb{E}\max(X-\theta,0)$ is infinitely differentiable as a function of $\theta$. Finally, the $1/i$ learning rate schedule in Assumption~\ref{assump:learningrate} is motivated by stochastic gradient algorithms, under which the parameter estimator achieves optimal convergence rates and admits a tractable limiting distribution \citep{zhang2016central}.

%Conditions similar to Assumption \ref{assump:kernelfun} is weaker than the kernelconditions commonly imposed in classical high-order kernel smoothing since our analysis only uses a first-order local approximation. 
%also differs from the smoothness assumptions in traditional kernel regression \citep{wand1994kernel}. Classical kernel smoothing typically assumes that the target function is $p$-times continuously differentiable with respect to the smoothing variable. However, such condition might be violated in setting where the policy class is parametrized by neural-network with non-smooth activations.

The following theorem upper bounds the bias and variance of the KAE value estimator. 

\begin{theorem}[Bias and variance of value estimator]\label{thm:value}
Suppose Assumptions \ref{assump:iidsampledprompt} -- \ref{assump:learningrate} hold. %and the learning rates $\eta_i \leq \beta/i$ for some constant $\beta>0$, the bandwidth $h$ satisfies $h\to0$ and $ih\to \infty$ as $i\to \infty$. 
Then for any $x$ that is sampled at the $i$th iteration and any bandwidth $h<1$,
{
\setlength{\abovedisplayskip}{8pt}
\setlength{\belowdisplayskip}{8pt}
\setlength{\abovedisplayshortskip}{3pt}
\setlength{\belowdisplayshortskip}{3pt}
\begin{eqnarray*}
    \text{Bias}(\widehat{V}_i^{(g)}(x)):= \mathbb{E}[\widehat{V}_i^{(g)}(x)] - V^{\pi_{\theta_i}}(x) = O(h) + O\Big(\frac{1}{i h}\Big),\quad 
    \text{Var}(\widehat{V}_i^{(g)}(x)) = O\Big(\frac{1}{N_i(x) h}\Big),
\end{eqnarray*}}
where $N_i(x)$ is the sample size for estimating $V^{\pi_{\theta_i}}(x)$, given by $G$ times the number of past occurrences of $x$ plus $G-1$ current samples (due to leave-one-out). The bias and variance are computed conditional on $x$ being sampled at iteration $i$ and on the past sampling history, under which $N_i(x)$ is treated as fixed.
\end{theorem}
Our results in Theorem~\ref{thm:value} are generally consistent with classical kernel smoothing theory, with the exception of an additional bias term of order $O((ih)^{-1})$. This term arises from approximating a continuous kernel integral by a discrete sum over the iteration index. 
Nevertheless, this term is of the same order as the variance under the resource-constrained setting where $N_i(x) \le Gi = O(i)$, and its square is therefore of higher order.
\begin{corollary}[Consistency of value estimator]\label{coro:value}
    Suppose the kernel bandwidth $h$ is proportional to $[N_i(x)]^{-1/3}$. 
    Under the conditions in Theorem \ref{thm:value} and in the resource-constrained setting where $G$ is fixed, KAE's value estimator achieves the convergence rate $\text{MSE}(\widehat{V}_i^{(g)}(x)) = O([N_i(x)]^{-2/3})$. In contrast, neither the GRPO nor the REINFORCE++ value estimator is consistent in this resource-constrained setting.
\end{corollary}

\begin{remark}
The rate in Corollary~\ref{coro:value} coincides with Stone's optimal
minimax rate for one-dimensional nonparametric regression with a Lipschitz continuous regression function. In contrast, the value estimators of GRPO and REINFORCE++ are both inconsistent in this resource-constrained setting. It thus formally verifies the advantage of KAE over both baseline algorithms in value estimation.
\end{remark}

We next analyze KAE’s gradient estimator $\widehat{g}_{\textrm{KAE}}(\theta)$; refer to its definition in Algorithm~\ref{alg1}. We begin with the following condition on the policy score. 
\begin{assumption}[Uncorrelatedness]\label{assump:score}
$
Z^{(b,g)}$ is uncorrelated with $\|\nabla_{\theta}\log \pi_{\theta}(Y^{(b,g)}|X^{(b)})\|^2$ given $X^{(b)}$. 
\end{assumption}
Assumption~\ref{assump:score} guarantees that the value function serves as the optimal baseline for minimizing the variance of the policy gradient estimator \citep{greensmith2004variance}. It justifies PPO and A2C, which use a value network for value function estimation, as well as our nonparametric statistical approach.
\begin{theorem}[MSE of gradient estimator]\label{thm:grad}
    Let $\widehat{g}_{\text{oracle}}(\theta)$ denote the oracle gradient estimator, with the oracle value function used to construct the advantage. 
    Under Assumptions \ref{assump:iidsampledprompt} -- \ref{assump:score}, we have
    {
\setlength{\abovedisplayskip}{8pt}
\setlength{\belowdisplayskip}{8pt}
\setlength{\abovedisplayshortskip}{3pt}
\setlength{\belowdisplayshortskip}{3pt}
\begin{eqnarray*}
        \text{MSE}(\widehat{g}_{\text{KAE}}(\theta_i))=\text{MSE}(\widehat{g}_{\text{oracle}}(\theta_i)) + O\left(\frac{h^{2}}{BG}\right)+O\left(\frac{1}{i^2 h^2BG}\right)  + O\left(\frac{m}{ihB^2G^2}\right)+O\left(\frac{m^2}{i^2h^2B^3G^2}\right).
    \end{eqnarray*}}
\end{theorem}
Theorem~\ref{thm:grad} upper bounds the difference in MSE between KAE and the oracle gradient estimator. In the resource-constrained setting, the MSE of the oracle gradient estimator is of order $O(B^{-1})$ \citep[Proposition 3]{zhou2026demystifying}. The first three extra error terms on the right-hand-side are proportional to the MSE of KAE's value estimator, each scaled by a factor of order $O(B^{-1})$. The last term is higher order and decays to zero more quickly as $ih$ or $B$ approaches infinity. 
This formally verifies that improved value estimation directly leads to improved gradient estimation.

Next, we consider an asymptotic regime in which the number of training iterations $i \to \infty$, while $G$ and $m$ are held fixed. The following oracle property is immediate.

\begin{corollary}[Oracle property of gradient estimator]\label{coro:grad}
    Under the same conditions in Theorem \ref{thm:grad}, if the kernel bandwidth satisfies $h\to 0$ and $i h\to \infty$ as $i\to\infty$, then the MSE of KAE's gradient estimator is asymptotically equivalent to that of the oracle  estimator. Moreover, it is smaller than MSEs of the GRPO and REINFORCE++ gradient estimators. 
\end{corollary}
The first part of Corollary \ref{coro:grad} establishes the asymptotic equivalence between KAE's gradient estimator and the oracle estimator. The second part highlights KAE's advantage over both baseline algorithms in gradient evaluation. 

Finally, we analyze the suboptimality gap of KAE's learned policy. For any policy $\pi$, define its suboptimality gap as $\Delta(\pi)=\sup_{\theta\in \Theta} \mathbb{E}^{\pi_{\theta}} (Z)-\mathbb{E}^{\pi} (Z)$. By definition, a smaller $\Delta(\pi)$ corresponds to a larger expected return, and hence a better policy. We impose the following assumptions for policy optimization. 
\begin{assumption}[Polyak-Lojasiewicz (PL) condition]\label{assump:PL}
There exists some constant $\mu>(2\beta)^{-1}$ such that $\|\nabla_{\theta} J(\theta)\|^2\ge 2\mu \Delta(\pi_{\theta})$ for any $\theta\in \Theta$. 
\end{assumption}

Assumption~\ref{assump:PL} is  frequently imposed for establishing convergence guarantees in nonconcave optimization, including in the analysis of deep learning models \citep{karimi2016linear}. It allows the objective function $J(\theta)$ to be nonconcave. To illustrate, suppose $\theta$ can be decomposed into  $(\theta_1^\top,\theta_2^\top)^\top$, and $J(\theta)$ depends only on $\theta_1$. In this case, $J(\theta)$ is not strictly concave, as its Hessian with respect to $\theta_2$ is zero. Nevertheless, the inequality in Assumption~\ref{assump:PL} can still hold, provided that $J(\theta)$ is concave in $\theta_1$. 
\begin{theorem}[Suboptimality gap of the policy]\label{thm:policy}
    Suppose Assumptions \ref{assump:iidsampledprompt} -- \ref{assump:PL} hold. Then for any $1\leq n_0 \le n$ and for $\varepsilon = 2\mu\beta-1$, we have
    {
    \setlength{\abovedisplayskip}{8pt}
    \setlength{\belowdisplayskip}{8pt}
    \setlength{\abovedisplayshortskip}{3pt}
    \setlength{\belowdisplayshortskip}{3pt}
    \begin{equation}\label{eqn:n0}
        \mathbb{E}[\Delta(\pi_{\theta_n})]\le \frac{c_1}{n^{1+\varepsilon}} +c_2 n_0\left(\frac{n_0}{n}\right)^{1+\varepsilon} + \frac{c_3}{n} \sup_{k\ge n_0}\text{MSE}(\widehat{g}_{\text{KAE}}(\theta_k)), 
    \end{equation}}
    where $c_1, c_2, c_3$ are positive constants depending only on $\beta$, $\mu$ and a smoothness constant.
\end{theorem}
Theorem \ref{thm:policy} upper bounds the suboptimality gap of KAE's learned policy. 
It shows that the MSE of KAE’s gradient estimator directly affects the suboptimality upper bound. This reinforces our motivation to improve value estimation, which in turn improves gradient evaluation and ultimately policy optimization. The following oracle property then establishes the asymptotic equivalence of the suboptimality upper bounds between KAE and the oracle algorithm, and shows that KAE achieves a suboptimality upper bound no larger than those of the two baseline algorithms. 
% shows that KAE achieves a no larger suboptimality upper bound than the two baseline algorithms. 
\begin{corollary}[Oracle property of the policy]\label{coro:policy}
    Suppose the assumptions in Theorem \ref{thm:policy} hold. If the bandwidth sequence $\{h_i\}_i$ indexed by the training step $i$, satisfies $h_i\to 0$ and $i h_i\to \infty$ as $i\to \infty$, then the suboptimality upper bound of KAE in \eqref{eqn:n0}, with $n_0$ growing to infinity with $n$, is asymptotically equivalent to that achieved by the oracle gradient estimator. Moreover, it is no larger than those achieved by GRPO and REINFORCE++.
\end{corollary}

\section{Experiments}\label{sec:experiment}
In this section, we conduct extensive experiments to evaluate the effectiveness of KAE against GRPO- and REINFORCE-type algorithms, as well as its robustness to the choice of bandwidth parameter and kernel function. We begin with a summary of our findings. Consistent with our theories, KAE is evaluated along the following three dimensions:
\begin{enumerate}[leftmargin=*]
    \item \textbf{Value estimation} (Section \ref{sec:exp-baseline-mse}): Across three training steps and two benchmark datasets, KAE achieves a {\bf 60}\%--{\bf 70}\% reduction in the MSE of the value estimator compared to GRPO, and an over {\bf 90}\% reduction compared to REINFORCE++ (Table \ref{tab:baseline_mse_cross}). Moreover, its MSE remains robust to the choice of bandwidth and kernel function (Figure~\ref{fig:mse_kernel_sensitivity}).  
    \item \textbf{Gradient evaluation} (Section \ref{sec:exp-gradvar}): The improvements in value estimation translate directly into more accurate gradient estimation. Specifically, across the same experimental settings, KAE achieves a {\bf 5}\%--{\bf 9}\% reduction in the MSE of the gradient estimator compared to GRPO, and a {\bf 32}\%--{\bf 65}\% reduction compared to REINFORCE++ (Table~\ref{tab:gradvar_stage_summary}).
    
    \item \textbf{Policy optimization} (Section \ref{sec:large-scale}): The gains in value and gradient estimation further translate into improved policy optimization. Compared to GRPO, Dr.~GRPO and GPG, KAE achieves the best performance in most scenarios, with an average improvement of {\bf 5\%} on MATH (Table~\ref{tab:qwen7b_math35}), and {\bf 11.8\%} on DAPO (Table \ref{tab:qwen7b_dapo17k}). In some cases, the improvement reaches up to {\bf 79.9\%}.
    It also stabilizes training and achieves improvements of up to {\bf 15\%} over REINFORCE-type algorithms (Figure~\ref{fig:qwen15b_n1_val_curves}). Finally, ablation studies confirm that these gains stem from improved value function estimation (Figures~\ref{fig:qwen15b_n1_val_curves} and \ref{fig:ablation_math_curves}).  
\end{enumerate}
\textbf{Implementation details}. Before detailing these results, we first describe the following technique employed in our implementation. Our Algorithm~\ref{alg1} is general and allows for any prompt sampling schedule in Line 3. To simplify our theoretical analysis, in Section \ref{sec:theory}, we assume that minibatches of prompts are sampled i.i.d. across training iterations (Assumption \ref{assump:iidsampledprompt}). Under this schedule, however, rewards associated with the same prompt may be too far apart in time to provide useful information for smoothing. To enable KAE to borrow information more effectively across training steps, we adopt the following sampling schedule: we partition all prompts into several minibatches, sample one minibatch without replacement at each training iteration, and then reuse the same minibatch for $J$ consecutive training steps. Additionally, we consider three benchmark reasoning datasets: GSM8K, MATH, and DAPO. The parameter $J$ is set to 10 for GSM8K and MATH, and to 8 for DAPO. In our main experiments, we use the triangular kernel $K(x)=2\max(1-x,0)$. We also consider the exponential kernel $K(x)=C\rho^{|x|}$ for some $0<\rho<1$ and $C>0$ when evaluating the sensitivity of KAE. Finally, we use Nadaraya-Watson estimators to normalize the value estimator to improve the performance when the sample size is small.

\subsection{Value estimation}
\label{sec:exp-baseline-mse}
\textbf{Experimental setup}. 
We first compare KAE’s value estimator with GRPO’s group mean estimator and REINFORCE++’s across-prompt reward average. For this evaluation, we consider two benchmark datasets: GSM8K and MATH. Following the  literature, we post-train a smaller Qwen2.5-1.5B-Instruct model on the simpler GSM8K dataset and a larger Qwen2.5-Math-7B model on the more challenging MATH dataset. For each dataset–model combination, we first apply the proposed prompt sampling scheme to fix the set of prompts used throughout training. We then post-train the model using KAE to obtain a sequence of parameter estimates $\{\theta_i\}_i$ across training steps. We aim to estimate the value function $V^{\pi_{\theta_i}}$ at three training steps: $i=10$, 90, and 170. 

\textbf{Evaluation}. Since the prompts are fixed, the MSE of the value estimators is computed with respect to the variability from the sampled completions. We approximate the ground-truth value function via Monte Carlo (MC) sampling by generating a large number of completions at each selected training step, and treat the resulting MC estimates as ground truth.

The MSEs of the GRPO and REINFORCE++ value estimators are straightforward to evaluate. At each selected training step, we generate $G$ completions per prompt and use these samples to compute the corresponding value estimators. We then repeat this procedure multiple times to estimate their MSEs. To ensure a fair comparison, we apply the same leave-one-out procedure to both GRPO and REINFORCE++ when estimating the value.
Additionally, although REINFORCE++ is designed to use only a single completion per prompt, we adapt the algorithm by using all $G$ completions to compute its average reward so that all methods utilize the same amount of data for value estimation.

However, evaluating the MSE of KAE's value estimator is more challenging, since the estimator's variability arises not only from sampling completions at the current training step, but also from those at previous steps. Ideally, this would require to sample past completions conditional on the current parameter $\theta_i$, which is intractable. To address this, we note that due to our use of a triangular kernel, the value estimator depends only on a finite number of previous steps (specifically, 4 steps) due to truncation. We therefore approximate the distribution of past completions using $\pi_{\theta_{i-1}}$, $\cdots$, $\pi_{\theta_{i-4}}$. This approximation is reasonable given the small learning rate ($10^{-6}$) used in our implementation, under which the model parameters change only minimally over a few training steps.

\begin{table}[t]
\centering
\caption{\small MSE ($\times 10^{-3}$; lower is better) of value estimators computed by KAE, GRPO and REINFORCE++ across three training steps and two datasets. For each training step, the MSE is first evaluated at the prompt level and then aggregated across prompts. The last two rows report the percentage reduction in MSE achieved by KAE compared to REINFORCE++ and GRPO.}
\label{tab:baseline_mse_cross}
\small
\setlength{\tabcolsep}{5pt}
\renewcommand{\arraystretch}{0.85}
\begin{tabular}{lcccccc}
\toprule
& \multicolumn{3}{c}{MATH (Qwen2.5-Math-7B)} & \multicolumn{3}{c}{GSM8K (Qwen2.5-1.5B)} \\
\cmidrule(lr){2-4}\cmidrule(lr){5-7}
Algorithm & Step 10 & Step 90 & Step 170 & Step 10 & Step 90 & Step 170 \\
\midrule
REINFORCE++ & $78.8$ & $89.2$ & $71.9$ & $73.2$ & $30.7$ & $30.9$ \\
GRPO & $19.4$ & $8.62$ & $5.07$ & $15.5$ & $5.65$ & $2.50$ \\
KAE & $\best{5.80}$ & $\best{2.67}$ & $\best{1.45}$ & $\best{4.13}$ & $\best{1.77}$ & $\best{0.91}$ \\
\midrule
Reduction vs. REINFORCE++ & $92.6\%$ & $97.0\%$ & $98.0\%$ & $94.4\%$ & $94.2\%$ & $97.1\%$ \\ 
Reduction vs. GRPO & $70.1\%$ & $69.0\%$ & $71.4\%$ & $73.3\%$ & $68.6\%$ & $63.6\%$ \\
\bottomrule
\end{tabular}
\end{table}

\textbf{Results}. Table~\ref{tab:baseline_mse_cross} reports the MSEs. We make two observations. First, KAE consistently achieves the smallest MSE in all cases, reducing GRPO’s MSE by {\bf 60}\%–{\bf 70}\% and REINFORCE++’s by over {\bf 90}\%. This highlights the effectiveness of leveraging historical observations to improve value function estimation. Moreover, these reductions are consistent across datasets and training steps, indicating that these gains are not specific to particular datasets or training steps. Second, the MSE of all value estimators decreases noticeably on GSM8K at Steps 90 and 170. This suggests that the model's response has become more certain at these steps, resulting in substantially smaller variance compared to Step 10.

\textbf{Sensitivity analysis}. We further conduct a sensitivity analysis on the MATH dataset to investigate the MSE of the proposed value estimator under different choices of bandwidth and kernel function. 
Figure~\ref{fig:mse_kernel_sensitivity} visualizes the MSEs at the three training steps. 
The results show that, for both kernel functions, triangular and exponential, and across all training steps and a wide range of bandwidth values, the MSE of KAE remains much smaller than that of GRPO and substantially smaller than that of REINFORCE++. This demonstrates the robustness of KAE with respect to both kernel bandwidth and kernel function.

\begin{figure}[t]
  \centering
  \includegraphics[width=\linewidth]{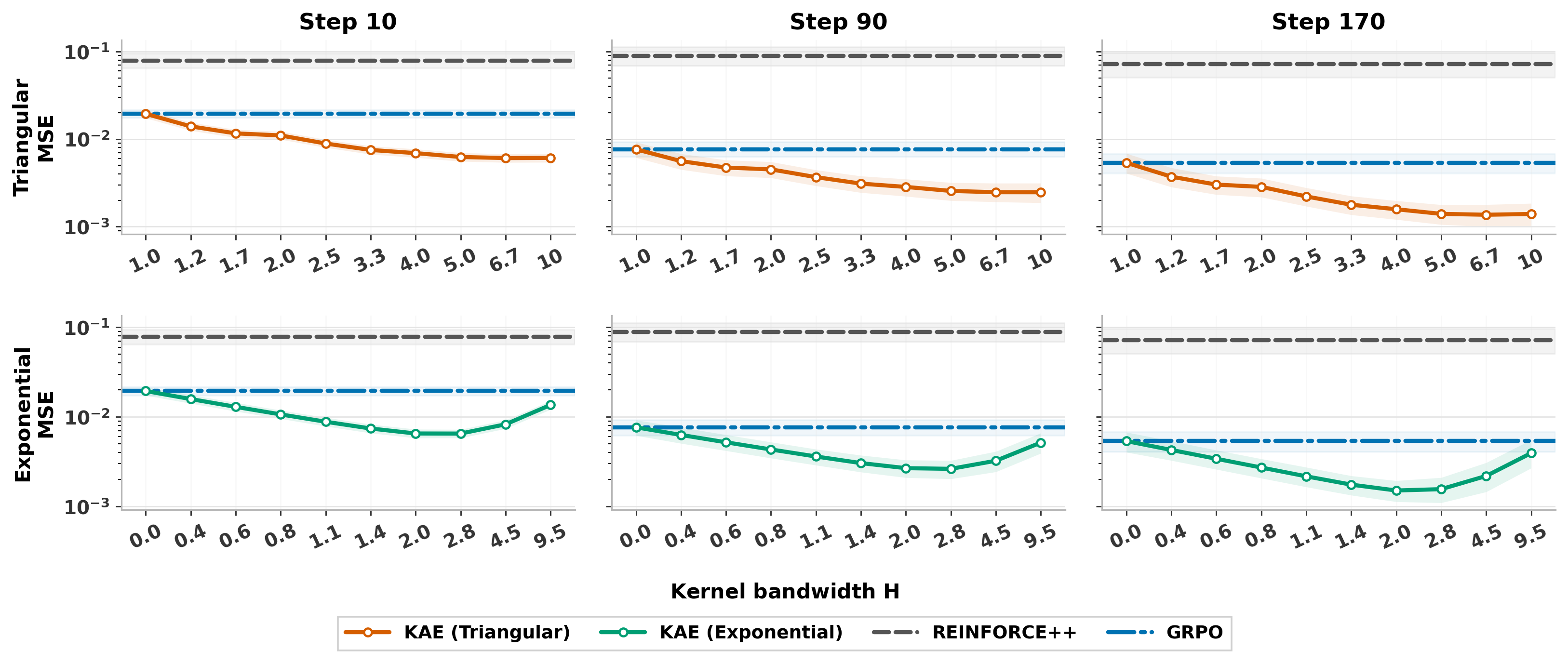}
  \caption{\small MSE of KAE’s value estimator on the MATH dataset across three training steps under varying kernel bandwidths. The left and right panels visualize the MSEs under the triangular and exponential kernels, respectively. Horizontal lines denote the MSEs of REINFORCE++ and GRPO, which are independent of bandwidth and kernel function.}
  \label{fig:mse_kernel_sensitivity}
\end{figure}

\subsection{Gradient evaluation}\label{sec:exp-gradvar}
Using the same datasets, MATH and GSM8K, we next compare KAE’s gradient estimators at the three training steps (10, 90, and 170) with those computed by GRPO and REINFORCE++. We similarly apply MC to evaluate the true gradients at these training steps. We then plug in each algorithm’s estimated value function to construct the corresponding advantage function and gradient estimator, and repeat this procedure multiple times to evaluate the gradient's MSE. For the MATH dataset, the model contains 7B parameters, making it memory intensive to store the full gradient vector. We therefore restrict attention to the subvector corresponding to the parameters in the final decoder block and the output layer when computing the MSE.

Table~\ref{tab:gradvar_stage_summary} summarizes the results. KAE again achieves the smallest MSE in all cases, with reductions ranging from {\bf 32.6}\%--{\bf 65.4}\% compared to REINFORCE++ and {\bf 5.4}\%--{\bf 8.5}\% compared to GRPO. The gains are consistent across datasets and training steps, highlighting the effectiveness of KAE for gradient evaluation.

\begin{table}[t]
\centering
\small
\caption{\small MSE ($\times 10^4$) of gradient estimators computed by KAE, GRPO and REINFORCE++ across three training steps and two datasets. The remaining details are the same as in Table~\ref{tab:baseline_mse_cross}.}
\label{tab:gradvar_stage_summary}
\setlength{\tabcolsep}{5pt}
\renewcommand{\arraystretch}{0.85}
\begin{tabular}{lcccccc}
\toprule
& \multicolumn{3}{c}{GSM8K (Qwen2.5-1.5B)} & \multicolumn{3}{c}{MATH (Qwen2.5-Math-7B)} \\
\cmidrule(lr){2-4}\cmidrule(lr){5-7}
Method & Step 10 & Step 90 & Step 170 & Step 10 & Step 90 & Step 170 \\
\midrule
REINFORCE++ & $11.98$ & $10.61$ & $15.60$ & $14.75$ & $3.88$ & $3.41$ \\
GRPO        & $7.62$  & $6.63$  & $6.31$  & $10.77$ & $1.59$ & $1.29$ \\
KAE         & $\best{7.00}$ & $\best{6.15}$ & $\best{5.97}$ & $\best{9.94}$ & $\best{1.46}$ & $\best{1.18}$ \\ 
\midrule
Reduction vs. REINFORCE++ & $41.57\%$ & $42.04\%$ & $61.73\%$ & $32.61\%$ & $62.37\%$ & $65.40\%$ \\
Reduction vs. GRPO        & $8.14\%$  & $7.24\%$  & $5.39\%$  & $7.71\%$  & $8.18\%$  & $8.53\%$ \\
\bottomrule
\end{tabular}
\end{table}

\subsection{Policy optimization}
\label{sec:large-scale}
\textbf{Experimental setup}. Finally, we demonstrate that more accurate value and gradient estimation indeed improve policy learning. To reflect our focus on the resource-constrained setting, we consider three group sizes $G\in \{1, 4, 8\}$. We evaluate different algorithms using three benchmark training datasets (GSM8K, MATH, and DAPO) by post-training three base models: Qwen2.5-1.5B-Instruct, Qwen2.5-Math-1.5B and Qwen2.5-Math-7B. 

More specifically, we consider two settings: single-stream and multi-stream. In the single-stream setting, we set $G=1$ and compare KAE against REINFORCE, since GRPO-type algorithms are not applicable when only a single completion is sampled per prompt. We consider two datasets in this setting, GSM8K and MATH. Since GSM8K is a simpler benchmark, we post-train the generic Qwen2.5-1.5B-Instruct model on this dataset. For MATH, we post-train the more task-specific Qwen2.5-Math-1.5B model. 
In the multi-stream setting, we choose $G\in \{4,8\}$ and compare KAE against GRPO-type algorithms, including the vanilla GRPO, Dr.~GRPO, and GPG. We again consider two  training datasets, DAPO and MATH, and post-train the larger Qwen2.5-Math-7B model on both datasets.

\textbf{Multi-stream setting}. Following standard practice, after training on DAPO and MATH, we  evaluate the post-trained models on a set of reasoning benchmarks covering both in-distribution and out-of-distribution datasets: AIME24, AIME25, AMC, MATH, Minerva, and Olympiad. For each benchmark, we report the accuracy measure (the percentage of questions answered correctly) for the post-trained model, together with the average accuracy across all benchmarks. The sampling temperature is fixed at 0.6 during evaluation.

\begin{table}[t]
\centering
\small
\caption{\small Accuracy of the Qwen2.5-Math-7B model post-trained on MATH (levels 3--5, 5k+). Results are reported at temperature $0.6$ after 200 training steps with group size $G=8$. The highest accuracy in each column is shown in \textbf{bold}, and the second-highest is \underline{underlined}.}
\label{tab:qwen7b_math35}
\setlength{\tabcolsep}{8pt}
\renewcommand{\arraystretch}{1}
\begin{tabular}{lccccccc}
\toprule
Method & AIME24 & AIME25 & AMC & MATH & Minerva & Olympiad & Avg \\
\midrule
Dr.~GRPO      & 0.2719 & 0.0875 & 0.5715 & 0.7580 & \underline{0.3309} & 0.3748 & 0.3991 \\
GPG         & \textbf{0.3052} & \underline{0.0896} & 0.5776 & \underline{0.7660} & 0.3125 & 0.3689 & \underline{0.4033} \\
GRPO        & 0.2844 & 0.0802 & \textbf{0.5858} & 0.7540 & 0.3162 & \underline{0.3852} & 0.4010 \\
KAE         & \underline{0.2969} & \textbf{0.0938} & \underline{0.5806} & \textbf{0.7720} & \textbf{0.3493} & \textbf{0.3970} & \textbf{0.4149} \\
\bottomrule
\end{tabular}
\end{table}

\begin{table}[t]
\centering
\small
\caption{\small Accuracy of the Qwen2.5-Math-7B model post-trained on DAPO (17k). Results are reported at temperature $0.6$ after 500 training steps with group size $G=4$. The highest accuracy in each column is shown in \textbf{bold}, and the second-highest is \underline{underlined}.}
\label{tab:qwen7b_dapo17k}
\setlength{\tabcolsep}{8pt}
\renewcommand{\arraystretch}{1}
\begin{tabular}{lccccccc}
\toprule
Method & AIME24 & AIME25 & AMC & MATH & Minerva & Olympiad & Avg \\
\midrule
GRPO      & \textbf{0.2833} & 0.1188 & 0.6261 & 0.7960 & \underline{0.3309} & 0.4178 & 0.4288 \\
Dr.~GRPO    & 0.2375 & 0.0990 & \underline{0.6408} & 0.7920 & 0.2757 & 0.4104 & 0.4092 \\
GPG       & \underline{0.2823} & \underline{0.1396} & 0.6205 & \textbf{0.8140} & \underline{0.3309} & \underline{0.4222} & \underline{0.4349} \\
KAE       & 0.2698 & \textbf{0.1781} & \textbf{0.6453} & \underline{0.8100} & \textbf{0.3456} & \textbf{0.4237} & \textbf{0.4454} \\
\bottomrule
\end{tabular}
\end{table}

Tables~\ref{tab:qwen7b_math35} and \ref{tab:qwen7b_dapo17k} report the accuracy results for the Qwen2.5-Math-7B model post-trained on MATH and DAPO, respectively. Across both training datasets, KAE achieves the highest accuracy in \textbf{5} out of \textbf{7} reasoning benchmarks, including the final average accuracy. On the remaining benchmarks, it typically attains the second-highest accuracy. Specifically, when trained on MATH, KAE improves over GRPO-type baselines by \textbf{0.8}\% to \textbf{17.0}\%, while when trained on DAPO, the improvement ranges from \textbf{0.4}\% to \textbf{79.9}\%. 
More importantly, these gains are observed across multiple benchmark datasets and multiple GRPO-type baselines, demonstrating the consistent improvements KAE delivers in policy optimization. 

\textbf{Single-stream setting}. Since GSM8K is a simpler dataset, the model's generated reasoning trajectories are much shorter than those for MATH, making training on GSM8K considerably less computationally demanding. We therefore repeat our experiments on GSM8K five times to improve the reproducibility of our results, whereas existing work typically reports only a single run. In contrast, since training on MATH is substantially more expensive, we report results from a single run in that setting.

\begin{figure}[t]
  \centering
  \includegraphics[width=0.8\linewidth]{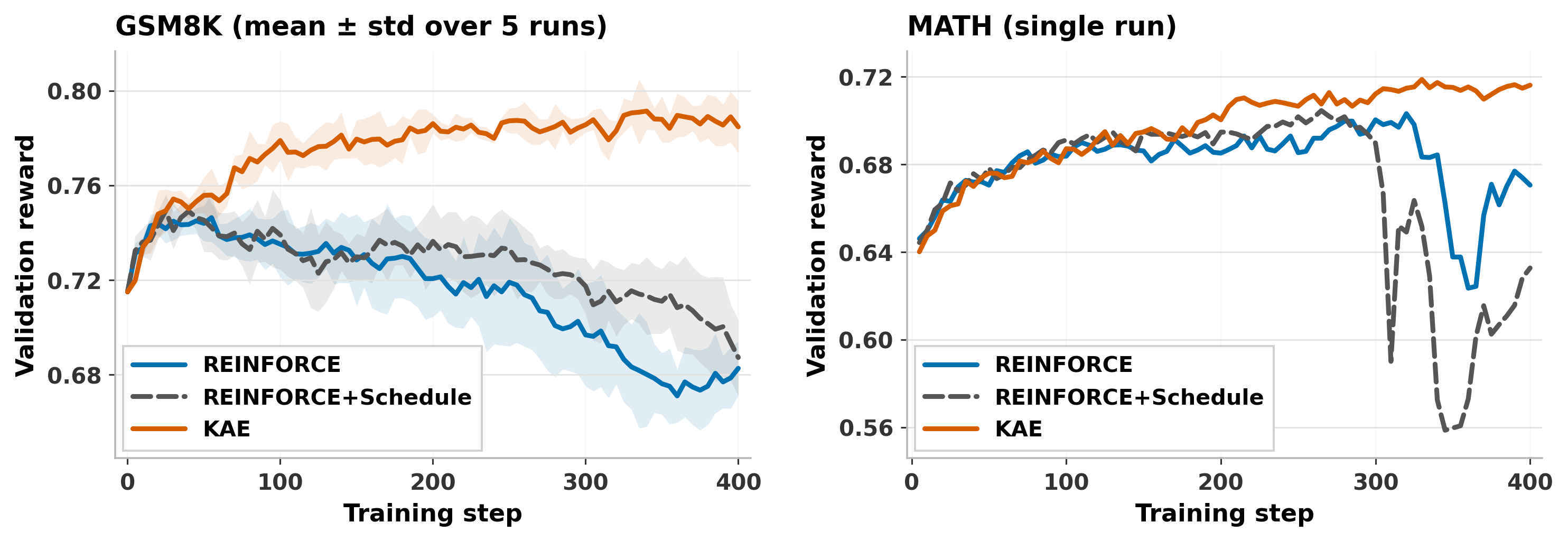}
  \caption{\small Test accuracy of models post-trained with standard REINFORCE (blue), KAE (red), and a REINFORCE variant using the proposed prompt sampling scheme, on GSM8K (left) and MATH (right) across different training steps. Shaded areas represent the standard error of the accuracy curves, aggregated over five training replications.}
  \label{fig:qwen15b_n1_val_curves}
\end{figure}

Figure~\ref{fig:qwen15b_n1_val_curves} visualizes the test accuracy of models post-trained with REINFORCE and the proposed KAE on both datasets. For GSM8K, where training is repeated five times, we additionally report the standard error estimated across these replications. It can be seen from Figure~\ref{fig:qwen15b_n1_val_curves} that, on both datasets, the accuracy of REINFORCE begins to decline after reaching its peak. We conjecture that this is due to the high variance of its gradient estimates, which makes training unstable. By contrast, the accuracy of KAE increases steadily throughout training. At the final training iteration, KAE improves over REINFORCE by {\bf 14.9}\% on GSM8K and {\bf 6.6}\% on MATH. These results show that more accurate value and gradient estimation indeed stabilize policy learning.

\textbf{Ablation study}. Our implemented KAE differs from GRPO- and REINFORCE-type algorithms in two ways: (i) it adopts a prompt sampling schedule that reuses each minibatch of prompts for multiple consecutive training steps, and (ii) it exploits historical rewards to improve value and advantage estimation. Our results above show that, when combined, these two components enable KAE to achieve better policy performance than the baseline methods. To isolate the contribution of the second component, we conduct an ablation study comparing KAE against variants of GRPO- and REINFORCE-type algorithms that adopt the proposed prompt sampling schedule, denoted by GRPO/REINFORCE + schedule. In this comparison, the primary difference is the value and advantage function estimation.

Figure~\ref{fig:qwen15b_n1_val_curves} visualizes the test accuracy of models trained by this variant of REINFORCE using GSM8K and MATH. It can be seen that the proposed prompt sampling schedule alone does not help much to stabilize training. The test accuracy still drops after reaching its peak across both datasets. Figure~\ref{fig:ablation_math_curves} compares the test accuracy of the schedule-matched GRPO variant with that of KAE on MATH. After roughly the first 30 training steps, KAE consistently achieves higher test accuracy than the GRPO variant. Taken together, these results indicate that more accurate value and advantage estimation translates into meaningful gains in downstream policy optimization.

\section{Conclusion}\label{sec:diss}
We introduce KAE, a kernel smoothing method for value and advantage function estimation in LLM reasoning, which can be seamlessly integrated into the subsequent policy optimization. Theoretically, we establish the oracle property of KAE and its superiority over GRPO-type algorithms in resource-constrained settings across three aspects: value estimation (Corollary \ref{coro:value}), gradient evaluation (Corollary \ref{coro:grad}) and policy learning (Corollary \ref{coro:policy}). Our empirical experiments further reinforce these theoretical findings: KAE achieves competitive performance against both GRPO- and REINFORCE-type algorithms (Tables \ref{tab:baseline_mse_cross} -- \ref{tab:qwen7b_dapo17k}, Figure~\ref{fig:qwen15b_n1_val_curves}), attains the oracle property in one-shot regimes (Figure \ref{fig:oneshot}), remains robust to the kernel bandwidth and kernel function (Figure \ref{fig:mse_kernel_sensitivity}), and its gains indeed come from the use of nonparametric methods for value and advantage estimation  (Figures~\ref{fig:qwen15b_n1_val_curves} and \ref{fig:ablation_math_curves}).

\bibliographystyle{plainnat}
\bibliography{bibliography}

@inproceedings{mnih2016asynchronous,
  title={Asynchronous methods for deep reinforcement learning},
  author={Mnih, Volodymyr and Badia, Adria Puigdomenech and Mirza, Mehdi and Graves, Alex and Lillicrap, Timothy and Harley, Tim and Silver, David and Kavukcuoglu, Koray},
  booktitle={International Conference on Machine Learning},
  pages={1928--1937},
  year={2016},
  organization={PMLR}
}

@inproceedings{robins2004optimal,
  title={Optimal structural nested models for optimal sequential decisions},
  author={Robins, James M},
  booktitle={Proceedings of the Second Seattle Symposium in Biostatistics: analysis of correlated data},
  pages={189--326},
  year={2004},
  organization={Springer}
}

@article{murphy2003optimal,
  title={Optimal dynamic treatment regimes},
  author={Murphy, Susan A},
  journal={Journal of the Royal Statistical Society Series B: Statistical Methodology},
  volume={65},
  number={2},
  pages={331--355},
  year={2003},
  publisher={Oxford University Press}
}

@article{xia2026statistical,
  title={A Statistical Framework for Alignment with Biased AI Feedback},
  author={Xia, Xintao and Xia, Zhiqiu and Zhang, Linjun and Cai, Zhanrui},
  journal={arXiv preprint arXiv:2602.08259},
  year={2026}
}

@article{liu2026reinforcement,
  title={Reinforcement Learning from Human Feedback: A Statistical Perspective},
  author={Liu, Pangpang and Shi, Chengchun and Sun, Will Wei},
  journal={arXiv preprint arXiv:2604.02507},
  year={2026}
}

@article{greensmith2004variance,
  title={Variance reduction techniques for gradient estimates in reinforcement learning},
  author={Greensmith, Evan and Bartlett, Peter L and Baxter, Jonathan},
  journal={Journal of Machine Learning Research},
  volume={5},
  number={Nov},
  pages={1471--1530},
  year={2004}
}

@article{williams1992simple,
  title={Simple statistical gradient-following algorithms for connectionist reinforcement learning},
  author={Williams, Ronald J},
  journal={Machine learning},
  volume={8},
  number={3},
  pages={229--256},
  year={1992},
  publisher={Springer}
}

@article{lai1985asymptotically,
  title={Asymptotically efficient adaptive allocation rules},
  author={Lai, Tze Leung and Robbins, Herbert},
  journal={Advances in applied mathematics},
  volume={6},
  number={1},
  pages={4--22},
  year={1985},
  publisher={Academic Press, Inc. Orlando, FL, USA}
}

@article{wei2022chain,
  title={Chain-of-thought prompting elicits reasoning in large language models},
  author={Wei, Jason and Wang, Xuezhi and Schuurmans, Dale and Bosma, Maarten and Xia, Fei and Chi, Ed and Le, Quoc V and Zhou, Denny and others},
  journal={Advances in neural information processing systems},
  volume={35},
  pages={24824--24837},
  year={2022}
}

@article{chu2025gpg,
  title={Gpg: A simple and strong reinforcement learning baseline for model reasoning},
  author={Chu, Xiangxiang and Huang, Hailang and Zhang, Xiao and Wei, Fei and Wang, Yong},
  journal={arXiv preprint arXiv:2504.02546},
  year={2025}
}

@article{robbins1951stochastic,
  title={A stochastic approximation method},
  author={Robbins, Herbert and Monro, Sutton},
  journal={The annals of mathematical statistics},
  pages={400--407},
  year={1951},
  publisher={JSTOR}
}

@article{ge2025reviewcausaldecisionmaking,
  title={A review of causal decision making},
  author={Ge, Lin and Cai, Hengrui and Wan, Runzhe and Xu, Yang and Song, Rui},
  journal={arXiv preprint arXiv:2502.16156},
  year={2025}
}

@book{tsiatis2019dynamic,
  title={Dynamic treatment regimes: Statistical methods for precision medicine},
  author={Tsiatis, Anastasios A and Davidian, Marie and Holloway, Shannon T and Laber, Eric B},
  year={2019},
  publisher={Chapman and Hall/CRC}
}

@article{kosorok2019precision,
  title={Precision medicine},
  author={Kosorok, Michael R and Laber, Eric B},
  journal={Annual review of statistics and its application},
  volume={6},
  number={1},
  pages={263--286},
  year={2019},
  publisher={Annual Reviews}
}

@article{chakraborty2013statistical,
  title={Statistical methods for dynamic treatment regimes},
  author={Chakraborty, Bibhas and Moodie, Erica E},
  journal={Springer-Verlag. doi},
  volume={10},
  number={978-1},
  pages={4--1},
  year={2013},
  publisher={Springer}
}

@article{stone1982optimal,
  title={Optimal global rates of convergence for nonparametric regression},
  author={Stone, Charles J},
  journal={The annals of statistics},
  pages={1040--1053},
  year={1982},
  publisher={JSTOR}
}

@article{huang1998projection,
  title={Projection estimation in multiple regression with application to functional ANOVA models},
  author={Huang, Jianhua Z},
  journal={The annals of statistics},
  volume={26},
  number={1},
  pages={242--272},
  year={1998},
  publisher={Institute of Mathematical Statistics}
}

@article{jaech2024openai,
  title={Openai o1 system card},
  author={Jaech, Aaron and Kalai, Adam and Lerer, Adam and Richardson, Adam and El-Kishky, Ahmed and Low, Aiden and Helyar, Alec and Madry, Aleksander and Beutel, Alex and Carney, Alex and others},
  journal={arXiv preprint arXiv:2412.16720},
  year={2024}
}

@inproceedings{james1961estimation,
  title={Estimation with quadratic loss},
  author={James, William and Stein, Charles and others},
  booktitle={Proceedings of the fourth Berkeley symposium on mathematical statistics and probability},
  volume={1},
  pages={361--379},
  year={1961},
  organization={University of California Press}
}

@article{shao2024deepseekmath,
  title={Deepseekmath: Pushing the limits of mathematical reasoning in open language models},
  author={Shao, Zhihong and Wang, Peiyi and Zhu, Qihao and Xu, Runxin and Song, Junxiao and Bi, Xiao and Zhang, Haowei and Zhang, Mingchuan and Li, YK and Wu, Yang and others},
  journal={arXiv preprint arXiv:2402.03300},
  year={2024}
}

@article{hu2025reinforce,
  title={Reinforce++: Stabilizing critic-free policy optimization with global advantage normalization},
  author={Hu, Jian and Liu, Jason Klein and Xu, Haotian and Shen, Wei},
  journal={arXiv preprint arXiv:2501.03262},
  year={2025}
}

@article{xiong2025a,
  title={A minimalist approach to llm reasoning: from rejection sampling to reinforce},
  author={Xiong, Wei and Yao, Jiarui and Xu, Yuhui and Pang, Bo and Wang, Lei and Sahoo, Doyen and Li, Junnan and Jiang, Nan and Zhang, Tong and Xiong, Caiming and others},
  journal={arXiv preprint arXiv:2504.11343},
  year={2025}
}

@inproceedings{
yan2025learning,
title={Learning to Reason under Off-Policy Guidance},
author={Jianhao Yan and Yafu Li and Zican Hu and Zhi Wang and Ganqu Cui and Xiaoye Qu and Yu Cheng and Yue Zhang},
booktitle={The Thirty-ninth Annual Conference on Neural Information Processing Systems},
year={2026},
}

@article{chai2025deep,
  title={Deep Transfer $Q$-Learning for Offline Non-Stationary Reinforcement Learning},
  author={Chai, Jinhang and Chen, Elynn and Fan, Jianqing},
  journal={arXiv preprint arXiv:2501.04870},
  year={2025}
}

@article{zhang2016central,
  title={CENTRAL LIMIT THEOREMS OF A RECURSIVE STOCHASTIC ALGORITHM WITH APPLICATIONS TO ADAPTIVE DESIGNS},
  author={Zhang, Li-Xin},
  journal={The Annals of Applied Probability},
  pages={3630--3658},
  year={2016},
  publisher={JSTOR}
}

@inproceedings{karimi2016linear,
  title={Linear convergence of gradient and proximal-gradient methods under the polyak-{\l}ojasiewicz condition},
  author={Karimi, Hamed and Nutini, Julie and Schmidt, Mark},
  booktitle={Joint European conference on machine learning and knowledge discovery in databases},
  pages={795--811},
  year={2016},
  organization={Springer}
}

@article{huang2026learning,
  title={On the Learning Dynamics of RLVR at the Edge of Competence},
  author={Huang, Yu and Wen, Zixin and Chi, Yuejie and Wei, Yuting and Singh, Aarti and Liang, Yingbin and Chen, Yuxin},
  journal={arXiv preprint arXiv:2602.14872},
  year={2026}
}

@article{miao2025reinforcement,
  title={Reinforcement learning for individual optimal policy from heterogeneous data},
  author={Miao, Rui and Shahbaba, Babak and Qu, Annie},
  journal={Annals of statistics},
  volume={53},
  number={4},
  pages={1513},
  year={2025}
}

@inproceedings{zhang2026dagmath,
title={{DAG}-Math: Graph-of-Thought Guided Mathematical Reasoning in {LLM}s},
author={Yuanhe Zhang and Ilja Kuzborskij and Jason D. Lee and Chenlei Leng and Fanghui Liu},
booktitle={The Fourteenth International Conference on Learning Representations},
year={2026}
}

@article{wang2025reinforcement,
  title={Reinforcement learning for reasoning in large language models with one training example},
  author={Wang, Yiping and Yang, Qing and Zeng, Zhiyuan and Ren, Liliang and Liu, Liyuan and Peng, Baolin and Cheng, Hao and He, Xuehai and Wang, Kuan and Gao, Jianfeng and others},
  journal={arXiv preprint arXiv:2504.20571},
  year={2025}
}

@inproceedings{
li2026disco,
title={Dis{CO}: Reinforcing Large Reasoning Models with Discriminative Constrained Optimization},
author={Gang Li and Ming Lin and Tomer Galanti and Zhengzhong Tu and Tianbao Yang},
booktitle={The Thirty-ninth Annual Conference on Neural Information Processing Systems},
year={2026},
}

@article{zheng2025group,
  title={Group sequence policy optimization},
  author={Zheng, Chujie and Liu, Shixuan and Li, Mingze and Chen, Xiong-Hui and Yu, Bowen and Gao, Chang and Dang, Kai and Liu, Yuqiong and Men, Rui and Yang, An and others},
  journal={arXiv preprint arXiv:2507.18071},
  year={2025}
}

@article{zhao2025geometric,
  title={Geometric-mean policy optimization},
  author={Zhao, Yuzhong and Liu, Yue and Liu, Junpeng and Chen, Jingye and Wu, Xun and Hao, Yaru and Lv, Tengchao and Huang, Shaohan and Cui, Lei and Ye, Qixiang and others},
  journal={arXiv preprint arXiv:2507.20673},
  year={2025}
}

@article{xu2025single,
  title={Single-stream policy optimization},
  author={Xu, Zhongwen and Ding, Zihan},
  journal={arXiv preprint arXiv:2509.13232},
  year={2025}
}

@article{zeng2025shrinking,
  title={Shrinking the Variance: Shrinkage Baselines for Reinforcement Learning with Verifiable Rewards},
  author={Zeng, Guanning and Zhou, Zhaoyi and Arora, Daman and Zanette, Andrea},
  journal={arXiv preprint arXiv:2511.03710},
  year={2025},
}

@article{hao2025on,
  title={On-policy rl with optimal reward baseline},
  author={Hao, Yaru and Dong, Li and Wu, Xun and Huang, Shaohan and Chi, Zewen and Wei, Furu},
  journal={arXiv preprint arXiv:2505.23585},
  year={2025}
}

@article{liu2025fin,
  title={Fin-r1: A large language model for financial reasoning through reinforcement learning},
  author={Liu, Zhaowei and Guo, Xin and Yang, Zhi and Lou, Fangqi and Zeng, Lingfeng and Niu, Jinyi and Li, Mengping and Qi, Qi and Liu, Zhiqiang and Han, Yiyang and others},
  journal={arXiv preprint arXiv:2503.16252},
  year={2025}
}

@article{guo2025deepseek,
  title={DeepSeek-R1 incentivizes reasoning in LLMs through reinforcement learning},
  author={Guo, Daya and Yang, Dejian and Zhang, Haowei and Song, Junxiao and Wang, Peiyi and Zhu, Qihao and Xu, Runxin and Zhang, Ruoyu and Ma, Shirong and Bi, Xiao and others},
  journal={Nature},
  volume={645},
  number={8081},
  pages={633--638},
  year={2025},
  publisher={Nature Publishing Group UK London}
}

@article{lambert2024tulu,
  title={Tulu 3: Pushing frontiers in open language model post-training},
  author={Lambert, Nathan and Morrison, Jacob and Pyatkin, Valentina and Huang, Shengyi and Ivison, Hamish and Brahman, Faeze and Miranda, Lester James V and Liu, Alisa and Dziri, Nouha and Lyu, Shane and others},
  journal={arXiv preprint arXiv:2411.15124},
  year={2024}
}

@article{li2025repo,
  title={Repo: Replay-enhanced policy optimization},
  author={Li, Siheng and Zhou, Zhanhui and Lam, Wai and Yang, Chao and Lu, Chaochao},
  journal={arXiv preprint arXiv:2506.09340},
  year={2025}
}

@article{lin2025cppo,
  title={Cppo: Accelerating the training of group relative policy optimization-based reasoning models},
  author={Lin, Zhihang and Lin, Mingbao and Xie, Yuan and Ji, Rongrong},
  journal={arXiv preprint arXiv:2503.22342},
  year={2025}
}

@inproceedings{
liu2025understanding,
title={Understanding R1-Zero-Like Training: A Critical Perspective},
author={Zichen Liu and Changyu Chen and Wenjun Li and Penghui Qi and Tianyu Pang and Chao Du and Wee Sun Lee and Min Lin},
booktitle={Second Conference on Language Modeling},
year={2025},
}

@article{gazi2026statisticalreinforcementlearningreal,
  title={Statistical Reinforcement Learning in the Real World: A Survey of Challenges and Future Directions},
  author={Gazi, Asim H and Guo, Yongyi and Gao, Daiqi and Xu, Ziping and Zhang, Kelly W and Murphy, Susan A},
  journal={arXiv preprint arXiv:2601.15353},
  year={2026}
}

@article{ouyang2022training,
  title={Training language models to follow instructions with human feedback},
  author={Ouyang, Long and Wu, Jeffrey and Jiang, Xu and Almeida, Diogo and Wainwright, Carroll and Mishkin, Pamela and Zhang, Chong and Agarwal, Sandhini and Slama, Katarina and Ray, Alex and others},
  journal={Advances in neural information processing systems},
  volume={35},
  pages={27730--27744},
  year={2022}
}

@article{yang2025qwen3,
  title={Qwen3 technical report},
  author={Yang, An and Li, Anfeng and Yang, Baosong and Zhang, Beichen and Hui, Binyuan and Zheng, Bo and Yu, Bowen and Gao, Chang and Huang, Chengen and Lv, Chenxu and others},
  journal={arXiv preprint arXiv:2505.09388},
  year={2025}
}

@article{fan1997local,
  title={Local polynomial regression: Optimal kernels and asymptotic minimax efficiency},
  author={Fan, Jianqing and Gasser, Theo and Gijbels, Ir{\`e}ne and Brockmann, Michael and Engel, Joachim},
  journal={Annals of the Institute of Statistical Mathematics},
  volume={49},
  number={1},
  pages={79--99},
  year={1997},
  publisher={Springer}
}

@article{nadaraya1964estimating,
  title={On estimating regression},
  author={Nadaraya, Elizbar A},
  journal={Theory of Probability \& Its Applications},
  volume={9},
  number={1},
  pages={141--142},
  year={1964},
  publisher={SIAM}
}

@article{chen2007large,
  title={Large sample sieve estimation of semi-nonparametric models},
  author={Chen, Xiaohong},
  journal={Handbook of econometrics},
  volume={6},
  pages={5549--5632},
  year={2007},
  publisher={Elsevier}
}

@article{stone1977consistent,
  title={Consistent nonparametric regression},
  author={Stone, Charles J},
  journal={The annals of statistics},
  pages={595--620},
  year={1977},
  publisher={JSTOR}
}

@book{grenander1981abstract,
  title={Abstract Inference},
  author={Grenander, Ulf},
  publisher={Wiley},
  year={1981}
}

@inproceedings{kool2019buy,
  title={Buy 4 REINFORCE Samples, Get a Baseline for Free!},
  author={Wouter Kool and Herke van Hoof and Max Welling},
  booktitle={ICLR 2019 Workshop on Deep Reinforcement Learning Meets Structured Prediction},
  year={2019},
}

@article{schulman2017proximal,
  title={Proximal policy optimization algorithms},
  author={Schulman, John and Wolski, Filip and Dhariwal, Prafulla and Radford, Alec and Klimov, Oleg},
  journal={arXiv preprint arXiv:1707.06347},
  year={2017}
}

@book{sutton2018reinforcement,
  title={ Reinforcement Learning: An Introduction },
  author={ Richard S. Sutton and Andrew G. Barto },
  year={ 2018 },
  edition={ 2 },
  publisher={ The MIT Press },
}

@article{zhou2026demystifying,
  title={Demystifying group relative policy optimization: Its policy gradient is a u-statistic},
  author={Zhou, Hongyi and Ye, Kai and Xu, Erhan and Zhu, Jin and Yang, Ying and Gong, Shijin and Shi, Chengchun},
  journal={arXiv preprint arXiv:2603.01162},
  year={2026}
}

@article{ertefaie2018constructing,
  title={Constructing dynamic treatment regimes over indefinite time horizons},
  author={Ertefaie, Ashkan and Strawderman, Robert L},
  journal={Biometrika},
  volume={105},
  number={4},
  pages={963--977},
  year={2018},
  publisher={Oxford University Press}
}

@article{luckett2020estimating,
  title={Estimating Dynamic Treatment Regimes in Mobile Health Using V-learning},
  author={Luckett, Daniel J and Laber, Eric B and Kahkoska, Anna R and Maahs, David M and Mayer-Davis, Elizabeth and Kosorok, Michael R},
  journal={Journal of the American Statistical Association},
  volume={115},
  number={530},
  pages={692},
  year={2020},
  publisher={NIH Public Access}
}

@article{liao2022batch,
  title={Batch policy learning in average reward markov decision processes},
  author={Liao, Peng and Qi, Zhengling and Wan, Runzhe and Klasnja, Predrag and Murphy, Susan A},
  journal={Annals of statistics},
  volume={50},
  number={6},
  pages={3364},
  year={2022},
  publisher={NIH Public Access}
}

@article{chen2024reinforcement,
  title={Reinforcement Learning in Latent Heterogeneous Environments},
  author={Chen, Elynn Y and Song, Rui and Jordan, Michael I},
  journal={Journal of the American Statistical Association},
  volume = {119},
  number = {548},
  pages = {3113--3126},
  year = {2024},
  publisher = {Taylor \& Francis},
}

@article{liang2023relative,
  title={Relative contrast estimation and inference for treatment recommendation},
  author={Liang, Muxuan and Yu, Menggang},
  journal={Biometrics},
  volume={79},
  number={4},
  pages={2920--2932},
  year={2023},
  publisher={Wiley Online Library}
}

@article{cho2026privacy,
  title={Privacy-Preserving Reinforcement Learning from Human Feedback via Decoupled Reward Modeling},
  author={Cho, Young Hyun and Sun, Will Wei},
  journal={arXiv preprint arXiv:2603.22563},
  year={2026}
}

@article{zhu2026align,
  title={ALIGN: Aligned Delegation with Performance Guarantees for Multi-Agent LLM Reasoning},
  author={Zhu, Tong and Chen, Baiting and Zhou, Jin and Zhou, Hua and Sankararaman, Sriram and Dai, Xiaowu},
  journal={arXiv preprint arXiv:2602.00127},
  year={2026}
}

@article{liu2025online,
  title={Online estimation and inference for robust policy evaluation in reinforcement learning},
  author={Liu, Weidong and Tu, Jiyuan and Chen, Xi and Zhang, Yichen},
  journal={The Annals of Statistics},
  volume={53},
  number={5},
  pages={2128--2152},
  year={2025},
  publisher={Institute of Mathematical Statistics}
}

@article{bian2025off,
  title={Off-policy evaluation in doubly inhomogeneous environments},
  author={Bian, Zeyu and Shi, Chengchun and Qi, Zhengling and Wang, Lan},
  journal={Journal of the American Statistical Association},
  volume={120},
  number={550},
  pages={1102--1114},
  year={2025},
  publisher={Taylor \& Francis}
}

@article{wang2023projected,
  title={Projected state-action balancing weights for offline reinforcement learning},
  author={Wang, Jiayi and Qi, Zhengling and Wong, Raymond KW},
  journal={The Annals of Statistics},
  volume={51},
  number={4},
  pages={1639--1665},
  year={2023},
  publisher={Institute of Mathematical Statistics}
}

@article{laber2014dynamic,
  title={Dynamic treatment regimes: Technical challenges and applications},
  author={Laber, Eric B and Lizotte, Daniel J and Qian, Min and Pelham, William E and Murphy, Susan A},
  journal={Electronic journal of statistics},
  volume={8},
  number={1},
  pages={1225},
  year={2014}
}

@article{li2024settling,
  title={Settling the sample complexity of model-based offline reinforcement learning},
  author={Li, Gen and Shi, Laixi and Chen, Yuxin and Chi, Yuejie and Wei, Yuting},
  journal={The Annals of Statistics},
  volume={52},
  number={1},
  pages={233--260},
  year={2024},
  publisher={Institute of Mathematical Statistics}
}

@article{lu2013variable,
  title={Variable selection for optimal treatment decision},
  author={Lu, Wenbin and Zhang, Hao Helen and Zeng, Donglin},
  journal={Statistical methods in medical research},
  volume={22},
  number={5},
  pages={493--504},
  year={2013},
  publisher={SAGE Publications Sage UK: London, England}
}

@article{shi2018high,
  title={High-dimensional A-learning for optimal dynamic treatment regimes},
  author={Shi, Chengchun and Fan, Alin and Song, Rui and Lu, Wenbin},
  journal={Annals of statistics},
  volume={46},
  number={3},
  pages={925},
  year={2018},
  publisher={NIH Public Access}
}

@article{shi2024statistically,
  title={Statistically efficient advantage learning for offline reinforcement learning in infinite horizons},
  author={Shi, Chengchun and Luo, Shikai and Le, Yuan and Zhu, Hongtu and Song, Rui},
  journal={Journal of the American Statistical Association},
  volume={119},
  number={545},
  pages={232--245},
  year={2024},
  publisher={Taylor \& Francis}
}

@article{shi2024value,
  title={Value enhancement of reinforcement learning via efficient and robust trust region optimization},
  author={Shi, Chengchun and Qi, Zhengling and Wang, Jianin g and Zhou, Fan},
  journal={Journal of the American Statistical Association},
  volume={119},
  number={547},
  pages={2011--2025},
  year={2024},
  publisher={Taylor \& Francis}
}

@article{zhou2024estimating,
  title={Estimating optimal infinite horizon dynamic treatment regimes via pt-learning},
  author={Zhou, Wenzhuo and Zhu, Ruoqing and Qu, Annie},
  journal={Journal of the American Statistical Association},
  volume={119},
  number={545},
  pages={625--638},
  year={2024},
  publisher={Taylor \& Francis}
}

@article{jin2025policy,
  title={Policy learning “without” overlap: Pessimism and generalized empirical Bernstein’s inequality},
  author={Jin, Ying and Ren, Zhimei and Yang, Zhuoran and Wang, Zhaoran},
  journal={The Annals of Statistics},
  volume={53},
  number={4},
  pages={1483--1512},
  year={2025},
  publisher={Institute of Mathematical Statistics}
}

@article{li2025reinforcement,
  title     = {Reinforcement Learning with Continuous Actions Under Unmeasured Confounding},
  author    = {Li, Yuhan and Han, Eugene and Hu, Yifan and Qi, Zhengling and Cui, Yifan and Zhu, Ruoqing},
  journal   = {Journal of the American Statistical Association},
  year      = {2025},
  publisher = {Taylor \& Francis},
  volume = {To appear}
}

@article{zhong2025risksensitive,
  title     = {Risk-Sensitive Deep RL: Variance-Constrained Actor-Critic Provably Finds Globally Optimal Policy},
  author    = {Zhong, Han and Deng, Xun and Fang, Ethan X. and Yang, Zhuoran and Wang, Zhaoran and Li, Runze},
  journal   = {Journal of the American Statistical Association},
  year      = {2025},
  publisher = {Taylor \& Francis},
  volume      = {To appear}
}

@article{lee2024lowrankcontextualreinforcementlearning,
  title={Low-rank contextual reinforcement learning from heterogeneous human feedback},
  author={Lee, Seong Jin and Sun, Will Wei and Liu, Yufeng},
  journal={arXiv preprint arXiv:2412.19436},
  year={2024}
}

@article{liu2025uncertaintyquantificationlargelanguage,
  title={Uncertainty Quantification for Large Language Model Reward Learning under Heterogeneous Human Feedback},
  author={Liu, Pangpang and Lu, Junwei and Sun, Will Wei},
  journal={arXiv preprint arXiv:2512.03208},
  year={2025}
}

@article{liu2025statistical,
  title={Statistical impossibility and possibility of aligning llms with human preferences: From condorcet paradox to nash equilibrium},
  author={Liu, Kaizhao and Long, Qi and Shi, Zhekun and Su, Weijie J and Xiao, Jiancong},
  journal={arXiv preprint arXiv:2503.10990},
  year={2025}
}

@article{lu2025contextual,
  title={Contextual online uncertainty-aware preference learning for human feedback},
  author={Lu, Nan and Fang, Ethan X and Lu, Junwei},
  journal={arXiv preprint arXiv:2504.19342},
  year={2025}
}

@article{xiao2025algorithmic,
  title     = {On the Algorithmic Bias of Aligning Large Language Models with {RLHF}: Preference Collapse and Matching Regularization},
  author    = {Xiao, Jiancong and Li, Ziniu and Xie, Xingyu and Getzen, Emily and Fang, Cong and Long, Qi and Su, Weijie},
  journal   = {Journal of the American Statistical Association},
  volume    = {120},
  number    = {552},
  pages     = {2154--2164},
  year      = {2025},
  publisher = {Taylor \& Francis}
}

@article{zhou2024federated,
  title={Federated offline reinforcement learning},
  author={Zhou, Doudou and Zhang, Yufeng and Sonabend-W, Aaron and Wang, Zhaoran and Lu, Junwei and Cai, Tianxi},
  journal={Journal of the American Statistical Association},
  volume={119},
  number={548},
  pages={3152--3163},
  year={2024},
  publisher={Taylor \& Francis}
}

@article{zheng2025parallel,
  title={Parallel-r1: Towards parallel thinking via reinforcement learning},
  author={Zheng, Tong and Zhang, Hongming and Yu, Wenhao and Wang, Xiaoyang and Dai, Runpeng and Liu, Rui and Bao, Huiwen and Huang, Chengsong and Huang, Heng and Yu, Dong},
  journal={arXiv preprint arXiv:2509.07980},
  year={2025}
}

@article{dai2025cde,
  title={Cde: Curiosity-driven exploration for efficient reinforcement learning in large language models},
  author={Dai, Runpeng and Song, Linfeng and Liu, Haolin and Liang, Zhenwen and Yu, Dian and Mi, Haitao and Tu, Zhaopeng and Liu, Rui and Zheng, Tong and Zhu, Hongtu and others},
  journal={arXiv preprint arXiv:2509.09675},
  year={2025}
}

@article{li2026bicc,
  title={When right meets wrong: Bilateral context conditioning with reward-confidence correction for grpo},
  author={Li, Yu and Lan, Tian and Qi, Zhengling},
  journal={arXiv preprint arXiv:2603.13134},
  year={2026}
}

@article{wang2025krpo,
  title={Kalman filter enhanced grpo for reinforcement learning-based language model reasoning},
  author={Wang, Hu and Ma, Congbo and Reid, Ian and Yaqub, Mohammad},
  journal={arXiv preprint arXiv:2505.07527},
  year={2025}
}

@article{han2026ebpo,
  title={EBPO: Empirical Bayes Shrinkage for Stabilizing Group-Relative Policy Optimization},
  author={Han, Kevin and Zhou, Yuhang and Gao, Mingze and Zhou, Gedi and Li, Serena and Kumar, Abhishek and Fan, Xiangjun and Li, Weiwei and Zhang, Lizhu},
  journal={arXiv preprint arXiv:2602.05165},
  year={2026}
}

\clearpage
\appendix

\newcommand{\AppendixPrefix}{A}
\renewcommand{\thepage}{\AppendixPrefix-\arabic{page}}
\renewcommand{\thesection}{\AppendixPrefix.\arabic{section}}
\renewcommand{\thesubsection}{\thesection.\arabic{subsection}}
\renewcommand{\theequation}{\AppendixPrefix.\arabic{equation}}
\renewcommand{\thefigure}{\AppendixPrefix.\arabic{figure}}
\renewcommand{\thetable}{\AppendixPrefix.\arabic{table}}

\setcounter{secnumdepth}{3}
\setcounter{tocdepth}{2}
\pagenumbering{arabic}

\section*{Appendicies for ``Kernelized Advantage Estimation: From Nonparametric Statistics to LLM Reasoning''}

\section{Proofs}
In this section, we provide the proofs of Theorem \ref{thm:value}, \ref{thm:grad}, \ref{thm:policy} and Corollary \ref{coro:value}, \ref{coro:grad}, \ref{coro:policy} stated in Section \ref{sec:theory}.

\subsection{Proof of Theorem \ref{thm:value}}
According to Line 10 of Algorithm \ref{alg1}, for a given $x\in\mathcal{X}$ sampled at $i$-th iteration (i.e., $X^{(b)} =x$ for some $b \in \{1,2,\ldots, B\}$ and we denote this event by $\mathcal{A}_i(x)$), its value estimator can be written as 
\begin{equation*}
    \widehat V_i^{(g)}(x)=
             \frac{1}{M_i(x)}\Big[\sum_{(I_j,Z_{j}) \in \mathcal{H}_i(x)} K\Big(\frac{i-I_j}{ih}\Big) Z_{j}+\sum_{k\neq g} K(0) Z^{(b,k)}\Big],
\end{equation*}
where $M_i(x)$ denotes the normalizing constant $h|\mathcal{H}_i(x)|+(G-1)K(0)$. Let $n_i(x)=|\mathcal{H}_i(x)|/G$ be the number of times where the prompt $x$ is sampled at iteration $0,1,\ldots, i-1$. Then under Assumption \ref{assump:iidsampledprompt}, it follows from Lemma \ref{lem:Uniform-distribution} that conditioned on $n_i(x)$, $I_j$ follows a uniform distribution over $\{0,1,\ldots,i-1\}$. Due to the independence between the prompts sampled at the $i$th iteration and historically sampled prompts, $I_j$ is also conditionally independent of $\mathcal{A}_i(x)$ given $n_i(x)$. Therefore,
\begin{eqnarray}\label{eqn:E-numerator}
    &&\mathbb{E}\Big[\sum_{(I_j,Z_{j}) \in \mathcal{H}_i(x)} K\Big(\frac{i-I_j}{ih}\Big) Z_{j}+\sum_{k\neq g} K(0) Z^{(b,k)}\Big| n_i(x),\mathcal{A}_i(x)\Big]\nonumber\\
    &&=Gn_i(x)\mathbb{E}\left[K\Big(\frac{i-I_j}{ih}\Big)V_{I_j}(x)\Big| n_i(x) \right]+K(0)(G-1)V_i(x)\nonumber\\
    &&=Gn_i(x)\sum_{t=1}^{i}\frac{1}{i}K\Big(\frac{t}{ih}\Big)V_{i-t}(x) + K(0)(G-1)V_i(x).
\end{eqnarray}
Let $\theta(t)$ be a continuous path interpolating $\theta_0,\theta_1,\ldots,\theta_i$, in the sense that $\theta(j)=\theta_j$ for each $j=0,\ldots,i$. Assume that $\theta(t)$ is differentiable %almost everywhere, except at the interpolation points $t=1,\ldots,i$, 
and that $\|\theta'(t)\| \leq K t^{-1}$
for some constant $K>0$ under Assumption \ref{assump:learningrate}. Lemma~\ref{lem:first-order-interpolation} guarantees the existence of such a path. Then under Assumption \ref{assump:psmooth}, the value function $V_t(x):= \mathbb{E}^{\pi_{\theta_t}}(R|X=x)$ is also Lipschitz continuous with respect to $t$ so that
\begin{eqnarray}\label{eqn:Lipcon}
    \max_{x\in \mathcal{X},t_2>t_1}|V_{t_1}(x)-V_{t_2}(x)| \le c \int_{t_1}^{t_2} \frac{1}{t}dt=c\log (t_2/t_1),
\end{eqnarray}
for some constant $c>0$. 

Therefore,
\begin{eqnarray}\label{eqn:Riemann-approx}
\begin{split}
    \sum_{t=1}^{i}\frac{1}{i}K\Big(\frac{t}{ih}\Big)V_{i-t}(x) 
    =&h\cdot\frac{1}{ih}\left[\sum_{t=1}^{[ih]+1} K\left(\frac{t}{ih}\right)V_{i-ih\cdot\frac{t}{ih}}(x)\right]\\
    =&h\left[\int_0^{1} K(u)V_{i-ihu}(x)du + O(1/ih)\right]\\
    %=&\frac{1}{i}\int_0^{ih} K\left(\frac{t}{ih}\right)V_{i-ih\cdot\frac{t}{ih}}(x)dt + O(1/i)\\
    =& h\int_0^{1} K(u)V_{i-ihu}(x)du + O(1/i),
\end{split}
\end{eqnarray}
where the first equality follows from the fact that $K$ is supported on $[0,1]$, the second equality follows from Lemma \ref{lem:Riemann-sum} (by setting $f(u)= K(u)V_{i-ihu}(x)$ and noticing that under Assumptions $\ref{assump:kernelfun}$ and \ref{assump:psmooth}, $f(u)$ is continuously differentiable on $[0,1]$ so that Lemma \ref{lem:Riemann-sum} is applicable). 

Moreover, %it follows from Assumption \ref{assump:kernelfun} and \eqref{eqn:Lipcon} that
\begin{eqnarray}\label{eqn:Taylor-1}
\begin{split}
    \int_0^1 K(t)V_{i-iht}(x)dt&=\int_0^1 K(t)V_x(i)\,dt+\int_0^1 K(t)\left[V_{i-iht}(x) - V_i(x)\right]  dt\\
    &=V_i(x) + O\Big(\int_0^1 \log(1/(1-ht))\,dt\Big) \\
    &=V_i(x)+ O(h).
\end{split}
\end{eqnarray}
where the second equality follows from Assumption \ref{assump:kernelfun} and \eqref{eqn:Lipcon}, and the last equality follows from the fact that $\log(1/(1-ht))\le ht/(1-ht)=O(ht)$, as $h<1$.

By definition, $N_i(x)=G n_i(x)+G-1=|\mathcal{H}_i(x)|+G-1$ given $\mathcal{A}_i(x)$. 
Combining equations \eqref{eqn:E-numerator}, \eqref{eqn:Riemann-approx} and \eqref{eqn:Taylor-1}, we obtain
\begin{eqnarray}
    \mathbb{E}[\widehat V_i^{(g)}(x)|N_i(x),\mathcal{A}_i(x)] &= &\frac{\left[h|\mathcal{H}_i(x)|+K(0)(G-1)\right]V_i(x) + O(N_i(x)/i) +O(N_i(x)h^2)}{h|\mathcal{H}_i(x)|+K(0)(G-1)}\nonumber\\
    &=&V_i(x) + O(h) + O\left(\frac{1}{ih}\right).
\end{eqnarray}
Therefore, $\text{Bias}(\widehat V_i^{(g)}(x)) = O(h) + O\left(\frac{1}{ih}\right)$. 

Next, we calculate the variance of value estimator. Notice that conditioned on sampled prompts, the rewards $Z_j$ are independent. Therefore, conditioned on $N_i(x)$ and $\mathcal{A}_i(x)$, we have
\begin{eqnarray}\label{eqn:value-var}
    \text{Var}(\widehat V_i^{(g)}(x)) &=& \frac{1}{M_i^2(x)}\left[\text{Var}\left(\sum_{(I_j,Z_{j}) \in \mathcal{H}_i(x)} K\Big(\frac{i-I_j}{ih}\Big) Z_{j}\right)+\text{Var}\left(\sum_{k\neq g} K(0) Z^{(b,k)})\right)\right]\nonumber\\
    &=&\frac{1}{M_i^2(x)}\left[|\mathcal{H}_i(x)|\text{Var}\left(K\Big(\frac{i-I_j}{ih}\Big) Z_{j}\right) + (G-1) K^2(0)\text{Var}(Z^{(b,k)})\right].
\end{eqnarray}
Under the bounded reward assumption (Assumption \ref{assump:boundedreward}), $\text{Var}(Z^{(b,k)})<\sigma^2$ for some constant $\sigma>0$ and
\begin{eqnarray}\label{eqn:kernel-var}
    \text{Var}\left(K\Big(\frac{i-I_j}{ih}\Big) Z_{j}\right)&\lesssim& \mathbb{E}\left(K^2\Big(\frac{i-I_j}{ih}\Big) \right)\nonumber\\
    &=& \sum_{t=0}^{i-1}\frac{1}{i}K^2\Big(\frac{t}{ih}\Big)\nonumber\\
    &=&\int_0^1 K^2(u/h)\,du + O\left(\frac{1}{ih}\right)\nonumber\\
    &=&h\int_0^1 K^2(u) \,du +  O\left(\frac{1}{ih}\right)\nonumber\\
    &=&O(h)+  O\left(\frac{1}{ih}\right).
\end{eqnarray}
Here, the second equality follows from Lemma \ref{lem:Uniform-distribution} that $I_j$ follows uniform distribution over $\{0,\ldots,i-1\}$ and the third equality follows from Lemma \ref{lem:Riemann-sum} by taking $f(x) = K^2(x/h)$. Combining equation \eqref{eqn:value-var} and \eqref{eqn:kernel-var}, we obtain
\begin{eqnarray}
    \text{Var}(\widehat V_i^{(g)}(x)) &\lesssim& \frac{|\mathcal{H}_i(x)|h+(G-1)K^2(0)}{[h|\mathcal{H}_i(x)| + (G-1)K(0)]^2} = O\left(\frac{1}{hN_i(x)}\right).
\end{eqnarray}
This completes the proof of Theorem \ref{thm:value}.

\subsection{Proof of Corollary \ref{coro:value}}

\textbf{Convergence rate of KAE:} By Theorem \ref{thm:value}, we obtain
\begin{equation*}
    \text{MSE}(\widehat V_i^{(g)}(x)) = \text{Bias}^2(\widehat V_i^{(g)}(x)) + \text{Var}(\widehat V_i^{(g)}(x)) = O(h^{2}) + O\left(\frac{1}{N_i(x)h}\right).
\end{equation*}
By taking the bandwidth $h = O(N_i^{-1/3})$, $\text{MSE}(\widehat V_i^{(g)}(x))$ is of the order $O(N_i^{-2/3})$ converges to $0$ as $N_i(x)\to \infty$.

\smallskip

\noindent \textbf{Convergence rate of GRPO:} The baseline function of GRPO algorithm is given by $(G-1)^{-1}\sum_{k\neq g} Z^{(b,k)}$, which is an unbiased estimator with variance of order $O(G^{-1})$. Therefore, the value estimator of GRPO algorithm is of order $O(G^{-1})$. Under the budget constrained setting where $G$ is fixed, the MSE of GRPO type estimator does not converge to $0$. 

\smallskip

\noindent \textbf{Convergence rate Reinforce++:} The baseline function of REINFORCE++ algorithm is given by $(BG)^{-1}\sum_{b=1}^B\sum_{k=1}^G  Z^{(b,k)}$. By direct calculation, the expectation of this value estimator equals $m^{-1}\sum_{j=1}^m V^{\pi_{\theta_i}}(x_j)$ and is generally different from $V^{\pi_{\theta_i}}(x)$. Thus, its MSE does not converge to $0$.

\subsection{Proof of Theorem \ref{thm:grad}}
Let $S_i(y,x):=\nabla_\theta \log \pi_{\theta_i}(y|x)$ and $V_i(x):=V^{\pi_{\theta_i}}(x)$.
For a prompt $X$ sampled at the $i$-th iteration, 
define
\begin{equation*}
    \widehat g_{\mathrm{KAE}}(X,\theta_i):=\frac1G\sum_{g=1}^GS_i(Y^{(g)},X)\left(Z^{(g)}-\widehat V_i^{(g)}(X)\right),
\end{equation*}
and
\begin{equation*}
    \widehat g_{\mathrm{oracle}}(X,\theta_i):=\frac1G\sum_{g=1}^GS_i(Y^{(g)},X)\left(Z^{(g)}-V_i(X)\right).
\end{equation*}
Then the corresponding gradient estimators can be written as
\begin{eqnarray}
    \widehat g_{\textrm{KAE}}(\theta_i)=\frac{1}{B}
    \sum_{b=1}^B \widehat{g}_{\textrm{KAE}}(X^{(b)},\theta_i).\nonumber\\
    \widehat g_{\textrm{oracle}}(\theta_i)=\frac{1}{B}
    \sum_{b=1}^B \widehat{g}_{\textrm{oracle}}(X^{(b)},\theta_i).\nonumber
\end{eqnarray}
Since $\mathbb E\left[S_i(Y,X)V_i(X)\mid X\right]=V_i(X)\mathbb E\left[S_i(Y,X)\mid X\right]=0,$ it is immediate to see that  $\widehat g_{\textrm{oracle}}(\theta_i)$ is unbiased to the true gradient
\begin{equation*}
    g(\theta_i):=\nabla_{\theta}J(\theta_i)=\frac1{|\mathcal X|}\sum_{x\in\mathcal X}\mathbb E\left[S_i(Y,x)Z\right].
\end{equation*}

We decompose KAE's MSE as follows:
\begin{eqnarray}\label{eqn:gradMSE-decomp}
\begin{split}
    \text{MSE}(\widehat{g}_{\textrm{KAE}}(\theta_i)) &= \mathbb{E}\left\Vert\widehat{g}_{\textrm{KAE}}(\theta_i) - \widehat{g}_{\textrm{oracle}}(\theta_i)+ \widehat{g}_{\textrm{oracle}}(\theta_i) - g(\theta_i)\right\Vert^2\\
    &=\text{MSE}(\widehat{g}_{\textrm{oracle}}(\theta_i)) +\mathbb{E}\left\Vert\widehat{g}_{\textrm{KAE}}(\theta_i) - \widehat{g}_{\textrm{oracle}}(\theta_i)\right\Vert^2\\
    &+2\mathbb{E}\left[\left(\widehat{g}_{\textrm{KAE}}(\theta_i) - \widehat{g}_{\textrm{oracle}}(\theta_i)\right)^\top\left(\widehat{g}_{\textrm{oracle}}(\theta_i) - g(\theta_i)\right)\right].
\end{split}
\end{eqnarray}

We first show that the interaction term is zero. 
For $x, \widetilde{x}\in\mathcal{X}, x\neq \widetilde{x}$, due to the independence between the response-reward pairs associated with $x$ and those associated with $\widetilde{x}$, as well as the unbiasedness of the oracle gradient estimator, we have
\begin{eqnarray*}
    &&\mathbb{E}\left\{\left(\widehat{g}_{\textrm{KAE}}(x,\theta_i) - \widehat{g}_{\textrm{oracle}}(x,\theta_i)\right)^\top \left(\widehat{g}_{\textrm{oracle}}(\widetilde{x},\theta_i)-g(\widetilde{x},\theta_i\right)\right\}\nonumber\\
    &=&\left\{\mathbb{E}\left(\widehat{g}_{\textrm{KAE}}(x,\theta_i) - \widehat{g}_{\textrm{oracle}}(x,\theta_i)\right)^\top\right\}\left\{\mathbb{E}\left(\widehat{g}_{\textrm{oracle}}(\widetilde{x},\theta_i)-g(\widetilde{x},\theta_i)\right)^\top\right\}=0.
\end{eqnarray*}
Thus, the interaction term equals 
\begin{equation*}
    \frac{2}{B}\mathbb{E}\left[\left(\widehat{g}_{\textrm{KAE}}(X, \theta_i) - \widehat{g}_{\textrm{oracle}}(X, \theta_i)\right)^\top(\widehat{g}_{\textrm{oracle}}(X, \theta_i)-g(\theta_i)) \right].
\end{equation*}
Let $\varepsilon_i^{(g)}(X):=\widehat V_i^{(g)}(X)-V_i(X)$. Then,
\[
    \widehat g_{\mathrm{KAE}}(X,\theta_i)-\widehat g_{\mathrm{oracle}}(X,\theta_i)
    =-\frac{1}{G}\sum_{g=1}^GS_i(Y^{(g)},X)\varepsilon_i^{(g)}(X).
\]
By the leave-one-out construction, $\varepsilon_i^{(g)}(X)$ is independent of
$(Y^{(g)},Z^{(g)})$ conditional on $X$ and the past history. Hence, $\widehat g_{\mathrm{KAE}}(\theta_i)$ is also an unbiased estimator for the true gradient. Thus, the interaction term can be simplified to 
\begin{equation}\label{eqn:simpleinteraction}
    \frac{2}{B}\mathbb{E}\left[\left(\widehat{g}_{\textrm{KAE}}(X,\theta_i) - \widehat{g}_{\textrm{oracle}}(X,\theta_i)\right)^\top\widehat{g}_{\textrm{oracle}}(X,\theta_i) \right].
\end{equation}
Under the uncorrelatedness condition in Assumption \ref{assump:score}, its conditional expectation given $X$ equals
\begin{align*}
    \frac{2}{B}&\mathbb E\left[\|S_i(Y^{(g)},X)\|^2\varepsilon_i^{(g)}(X)\left(Z^{(g)}-V_i(X)\right)\mid X\right] \\
&\qquad =\frac{2}{B}\mathbb E\left[\varepsilon_i^{(g)}(X)\mid X\right]\mathbb{E}
    \|S_i(Y^{(g)},X)\|^2\mathbb E\left[\left(Z^{(g)}-V_i(X)\right)\mid X
\right]=0.
\end{align*}
This verifies that \eqref{eqn:simpleinteraction} equals zero as well. 
Consequently,
\[
    \mathbb E
    \left[
    \left(
    \widehat g_{\mathrm{KAE}}(\theta_i)
    -
    \widehat g_{\mathrm{oracle}}(\theta_i)
    \right)^\top
    \left(
    \widehat g_{\mathrm{oracle}}(\theta_i)
    -
    g(\theta_i)
    \right)
    \right]
    =0.
    \label{eq:interaction-zero-polished}
\]

It remains to upper bound $\mathbb E\left\|\widehat g_{\mathrm{KAE}}(\theta_i)-\widehat g_{\mathrm{oracle}}(\theta_i)\right\|^2$. Define
\[
    R_b
    :=
    \widehat g_{\mathrm{oracle}}(\theta_i)-\widehat g_{\mathrm{KAE}}(\theta_i)
    =\frac1G\sum_{g=1}^G
    S_i(Y^{(b,g)},X^{(b)})
    \varepsilon_i^{(g)}(X^{(b)}).
\]
By the leave-one-out construction and the property of the score function, we have $\mathbb E(R_b| X^{(b)})=0$. Thus, the cross terms across different prompts vanish, and
\begin{eqnarray}\label{eqn:mse-main}
\begin{split}
    &\mathbb{E}\left\|\widehat g_{\mathrm{KAE}}(\theta_i)-\widehat g_{\mathrm{oracle}}(\theta_i)\right\|^2\\
    =& \frac1B\mathbb E\|R_b\|^2\\
    =&\frac{1}{BG^2}\mathbb{E}\left\Vert\sum_{g=1}^GS_i(Y^{(b,g)},X^{(b)})\varepsilon_i^{(g)}(X^{(b)})\right\Vert^2\\
    =&\frac{1}{BG}\mathbb{E}\left[\Vert S_i(Y^{(g)},X)\Vert^2(\varepsilon_i^{(g)}(X))^2\right]\\
    +&\frac{1}{BG^2}\sum_{k\neq l}\mathbb{E}\left[S_i(Y^{(k)},X)^\top S_i(Y^{(l)},X)\varepsilon_i^{(k)}(X)\varepsilon_i^{(l)}(X)\right].\\
    :=& I_1 + I_2.
\end{split}
\end{eqnarray}
%Notice that, conditional on $X$, the value estimator is a weighted sum of the leave-one-out rewards and historical rewards, both of which are independent of $Y^{(g)}$. Therefore, $\Vert S_i(Y^{(g)},X)\Vert^2$ is conditionally independent of $\varepsilon_i^{(g)}(X)$. 
It follows from the bounded score condition in Assumption \ref{assump:boundedreward} that
\begin{equation}\label{eqn:I1}
    I_1 = \frac{1}{BG}\mathbb{E}\Vert\nabla_{\theta}\log \pi_{\theta_i}(Y^{(g)}|X)\Vert^2 
[\text{MSE}(\widehat{V}_i^{(g)}(X))]=O\Big(\frac{1}{BG}\mathbb{E}[\text{MSE}(\widehat{V}_i^{(g)}(X))]\Big).
\end{equation}
According to Theorem \ref{thm:value}, we have that 
\begin{equation*}
    \text{MSE}(\widehat{V}_i^{(g)}(X)) \lesssim h^{2} + \frac{1}{i^2 h^2}+ \frac{1}{N_i(X)h}, 
\end{equation*}
where $a_i\lesssim b_i$ indicates the existence of some constant $c>0$ such that $a_i\le cb_i$ for all $i$, and $N_i(X)=G n_i(X)+G-1$, with $n_i(X)$ being the number of previous occurrences of prompt $X$ before iteration $i$.

Under Assumption \ref{assump:iidsampledprompt}, it is straightforward to show that
$n_i(X)\sim \mathrm{Bin}(i,B/m)$. Without loss of generality, assume $G\ge 2$. We can write
\[
    \frac1{N_i(X)}
    =
    \frac1{G[n_i(X)+(G-1)/G]}.
\]
Applying Lemma \ref{lem:binomial-bound} with
$a=(G-1)/G$ gives 
\begin{eqnarray}\label{eqn:expectednix}
    \mathbb E\left[\frac1{N_i(X)}|X\right]=O\left(\frac{m}{iBG}\right).
\end{eqnarray}
Therefore,
\[
    \mathbb E
    \left[
    \mathrm{MSE}\left(\widehat V_i^{(g)}(X)\right)|X
    \right]
    =
    O(h^{2})+O\left(\frac{1}{i^2h^2}\right)
    +
    O\left(\frac{m}{ihBG}\right).
\]
In view of equation \eqref{eqn:I1}, we obtain that
\begin{equation}\label{eqn:I1-order}
    I_1 =  O\left(\frac{h^{2}}{BG}\right)  + O\left(\frac{1}{i^2h^2 BG}\right)+O\left(\frac{m}{ihB^2G^2}\right).
\end{equation}

Next, we study $I_2$. Conditioned on $X$, $Y^{(k)}$ is independent of $Y^{(l)}$ and past data history, and hence being independent of $\theta_i$, $V_i(X)$, $\varepsilon_i^{(k)}(X)$ and $S_i(Y^{(l)},X)$. Consequently, $S_i(Y^{(k)},X)$ is conditionally uncorrelated with these quantities. The only dependence arises through the reward $Z^{(k)}$, which enters the construction of $\varepsilon_i^{(l)}(X)$ with weight proportional to $K(0)/M_i(X)$. This leads to
\begin{eqnarray*}
    I_2=\frac{1}{BG^2}\sum_{k\neq l}\mathbb{E}\left[S_i(Y^{(k)},X)^\top S_i(Y^{(l)},X)\varepsilon_i^{(k)}(X)\frac{K(0)Z^{(k)}}{M_i(X)}\right].
\end{eqnarray*}
By the same argument, $S_i(Y^{(l)},X)$ is correlated only with the reward $Z^{(l)}$ appearing in $\varepsilon_i^{(k)}$ inside the expectation above. Therefore, $I_2$ can be simplified to
\begin{eqnarray}
    I_2 &=& \frac{1}{BG^2}\sum_{k\neq l}\mathbb{E}\left[S_i(Y^{(k)},X)^\top S_i(Y^{(l)},X)\frac{K^2(0)Z^{(k)}Z^{(l)}}{M_i^2(X)}\right].\nonumber
\end{eqnarray}
Using the bounded policy score condition again (Assumption \ref{assump:boundedreward}) and similar arguments to the proof of \eqref{eqn:expectednix}, we obtain
\begin{eqnarray}\label{eqn:I2-order}
    I_2\lesssim \frac{1}{BG^2}\times G(G-1)\times\mathbb{E}\left[\frac{1}{M_i(X)^2}\right] = O\left(\frac{m^2}{i^2 h^2 B^3 G^2}\right).
\end{eqnarray}
Combining \eqref{eqn:gradMSE-decomp}, \eqref{eqn:mse-main},\eqref{eqn:I1-order} and \eqref{eqn:I2-order}, we obtain
\begin{eqnarray*}
        \text{MSE}(\widehat{g}_{\text{KAE}}(\theta_i))=\text{MSE}(\widehat{g}_{\text{oracle}}(\theta_i)) + O\left(\frac{h^{2}}{BG}\right)  + O\left(\frac{1}{i^2h^2BG}\right)+O\left(\frac{m}{ihB^2G^2}\right)+O\left(\frac{m^2}{i^2 h^2 B^3 G^2}\right).
\end{eqnarray*}
This completes the proof.

\subsection{Proof of Corollary \ref{coro:grad}}
By Theorem \ref{thm:grad},
\[
    \mathrm{MSE}(\widehat g_{\mathrm{KAE}}(\theta_i))
    =
    \mathrm{MSE}(\widehat g_{\mathrm{oracle}}(\theta_i))
    +
    O\left(\frac{h^{2}}{BG}\right)
    +
    O\left(\frac{1}{i^2h^2BG}\right)
    +
    O\left(\frac{m}{ihB^2G^2}\right)+O\left(\frac{m^2}{i^2 h^2 B^3 G^2}\right).
\]
Under the conditions that $h\to 0$ and $ih\to\infty$, and $G$ and $m$ are held fixed, we have
\[
    \mathrm{MSE}(\widehat g_{\mathrm{KAE}}(\theta_i))
    =
    \mathrm{MSE}(\widehat g_{\mathrm{oracle}}(\theta_i))
    +
    o(B^{-1}).
\]
Since the oracle gradient estimator has MSE of order $\Omega(B^{-1})$ \citep{zhou2026demystifying}, it follows
that KAE's gradient estimator is asymptotically equivalent to the oracle
gradient estimator.

It remains to compare KAE with GRPO and REINFORCE++. 
For any baseline
$\widetilde V_i(X)$ that is independent of the current response conditional on
$X$,
\[
\begin{aligned}
&\mathbb E\left[
    \left\|
    S_i(Y,X)\left(Z-\widetilde V_i(X)\right)
    \right\|^2
    \mid X
\right] \\
&\quad =
\mathbb E\left[
    \left\|
    S_i(Y,X)\left(Z-V_i(X)\right)
    \right\|^2
    \mid X
\right]
+
\mathbb E\left[
    \|S_i(Y,X)\|^2\mid X
\right]
\left(\widetilde V_i(X)-V_i(X)\right)^2,
\end{aligned}
\]
where the cross term equals zero by Assumption \ref{assump:score}. Hence, the
excess MSE of a gradient estimator using a non-oracle baseline is controlled by
the MSE of its value estimator. 

By Corollary \ref{coro:value}, neither GRPO, nor REINFORCE++ produces consistent value estimators in the resource-constrained regime. Consequently, they have asymptotically no smaller MSE than the oracle
estimator, and in the nondegenerate case strictly larger asymptotic MSE.
Therefore, KAE's gradient estimator is asymptotically equivalent to the oracle
estimator and has no larger asymptotic MSE than the GRPO and REINFORCE++
gradient estimators. The proof is hence completed.

\subsection{Proof of Theorem \ref{thm:policy}}
Recall that $J(\theta) = \mathbb{E}^{\pi_\theta}(Z)$ defined over $\theta \in \Theta$ whose gradient is given by $g(\theta):=\nabla_{\theta} J(\theta)$. Under
Assumption \ref{assump:psmooth}, $J$ is $L$-smooth where $L$ denotes the Lipschitz constant that $\|g(\theta_1)-g(\theta_2)\|_2\le L\|\theta_1-\theta_2\|_2$. Therefore, for the update $\theta_{n+1} = \theta_n+\eta_n\widehat{g}_{\mathrm{KAE}}(\theta_n)$, we have 
    \begin{equation}\label{eqn:iter1}
        J(\theta_{n+1}) = J(\theta_n+\eta_n\widehat{g}_{\mathrm{KAE}}(\theta_n))\geq J(\theta_n) +\eta_n g^\top(\theta_n)\widehat{g}_{\mathrm{KAE}}(\theta_n) - \frac{1}{2}L\eta_n^2\Vert\widehat{g}_{\mathrm{KAE}}(\theta_n)\Vert^2.
    \end{equation}
    Let $\mathcal F_n$ denote the sigma-field generated by $\theta_n$ and all
    randomness before the current minibatch is sampled. By the leave-one-out
    construction, $\widehat V_n^{(g)}(X)$ is independent of the current response
    $Y^{(g)}$ conditional on $X$ and $\mathcal F_n$. Moreover, $\mathbb E\left[\nabla_\theta\log\pi_{\theta_n}(Y^{(g)}| X)| X,\mathcal F_n\right]=0$. Therefore, 
    \begin{align*}
    \mathbb E[\widehat g_{\mathrm{KAE}}(\theta_n)| \mathcal F_n]
    &=\mathbb E^{\pi_{\theta_n}}\left[\nabla_\theta\log\pi_{\theta_n}(Y|X)\left\{Z-\widehat V_n^{(g)}(X)\right\}|\mathcal F_n\right]  \\
    &=\mathbb E^{\pi_{\theta_n}}\left[\nabla_\theta\log\pi_{\theta_n}(Y\mid X)Z\mid \mathcal F_n\right] \\
    &=g(\theta_n).
\end{align*}
    Taking conditional expectation on both sides of \eqref{eqn:iter1}, we obtain
    \begin{eqnarray}\label{eqn:someeqn3}
        \mathbb{E}[J(\theta_{n+1})|\mathcal{F}_n] &\geq& J(\theta_n) +\eta_n\Vert g(\theta_n)\Vert^2 - \frac{1}{2}L\eta_n^2\mathbb{E} \bigr[\Vert\widehat{g}(\theta_n)\Vert^2|\mathcal{F}_n\bigr]. 
    \end{eqnarray}
    Since $\widehat g_{\mathrm{KAE}}(\theta_n)$ is conditionally unbiased given $\mathcal{F}_n$ (as shown in the proof of Theorem \ref{thm:grad}, which follows from the independence between $Y^{(g)}$ and other quantities such as $\theta_n$ and $\widehat{V}^{(g)}(X)$ given $X$), the last term on the RHS can be represented by $-L\eta_n^2 \big[\|g(\theta_n)\|^2 + \mathbb{E}\left\{\Vert\widehat{g}_{\text{KAE}}(\theta_n) - g(\theta_n)\Vert^2|\mathcal{F}_n\right\}\big] /2$. 
    
    Let $\theta^*$ be a maximizer of $J(\theta)$. Then, rearranging the terms in \eqref{eqn:someeqn3} yields:
    \begin{eqnarray}
        J(\theta^*) - \mathbb{E}[J(\theta_{n+1})|\mathcal{F}_n] &\leq& J(\theta^*) -J(\theta_n) - \eta_n\Vert g(\theta_n)\Vert^2 \nonumber\\
        &&\qquad +\frac{1}{2}L\eta_n^2\Vert g(\theta_n)\Vert^2 +\frac{1}{2}L\eta_n^2\mathbb{E}\left\{\Vert\widehat{g}_{\text{KAE}}(\theta_n) - g(\theta_n)\Vert^2|\mathcal{F}_n\right\}.\nonumber %\text{MSE}(\widehat{g}_n(\theta_n)).\nonumber
    \end{eqnarray}
    It follows that
    \begin{equation*}
        \mathbb{E}[\Delta(\pi_{\theta_{n+1}})|\mathcal{F}_n]\leq \Delta(\pi_{\theta_n}) - \left(\eta_n - \frac{L\eta_n^2}{2}\right)\Vert  g(\theta_n)\Vert^2 + \frac{L\eta_n^2}{2}\mathbb{E}\left\{\Vert\widehat{g}_{\text{KAE}}(\theta_n) - g(\theta_n)\Vert^2|\mathcal{F}_n\right\}
    \end{equation*}
    Under the PL condition (Assumption \ref{assump:PL}), $\|g(\theta_n)\|^2
    \ge 2\mu\Delta(\pi_{\theta_n})$. Therefore, 
    \begin{equation}\label{eqn:nonasymp-iter}
        \mathbb{E}[\Delta(\pi_{\theta_{n+1}})|\mathcal{F}_n]\leq \left(1 - 2\mu\eta_n +\mu L\eta_n^2\right)\Delta(\pi_{\theta_{n}}) +\frac{L\eta_n^2}{2}\mathbb{E}\left\{\Vert\widehat{g}_{\text{KAE}}(\theta_n) - g(\theta_n)\Vert^2|\mathcal{F}_n\right\}.
    \end{equation}
    Take expectation on both sides of inequality \eqref{eqn:nonasymp-iter} and take $\eta_n = \beta/n$ as suggested in Assumption \ref{assump:learningrate}, we obtain the following recursion:
    \begin{equation*}
        \mathbb{E}[\Delta(\pi_{\theta_{n+1}})]\leq \left(1 - \frac{2\mu\beta}{n} +\frac{\mu L\beta^2}{n^2}\right)\mathbb{E}[\Delta(\pi_{\theta_{n}})] +\frac{L\beta^2}{2n^2}\text{MSE}(\widehat{g}_{\text{KAE}}(\theta_n) ).
    \end{equation*}
    To simplify notation, we define $C_n = \text{MSE}(\widehat{g}_{\text{KAE}}(\theta_n))$. Moreover, since $L$ can be taken larger than $\mu$, it follows that $1-2\mu\beta /n + \mu L\beta^2/n^2 \geq (n-\mu\beta)^2/n^2\geq0$ for all $n$. Unrolling the recursion yields
    \begin{eqnarray}
        \mathbb{E}[\Delta(\pi_{\theta_{n}})]\leq \prod_{k=1}^n\left(1 - \frac{2\mu\beta}{k} +\frac{\mu L\beta^2}{k^2}\right)\Delta(\pi_{\theta_0}) + \sum_{k=1}^n\frac{L\beta^2 C_k}{2k^2}\prod_{j=k+1}^n\left(1 - \frac{2\mu\beta}{j} +\frac{\mu L\beta^2}{j^2}\right).\nonumber
    \end{eqnarray}
    Under Assumption \ref{assump:PL},  $2\mu\beta>1$. Thus, applying Lemma \ref{lem:uniform-product-bound}, we obtain there exists some universal constant $K = K(\mu,\beta,L)$ only depends on $\mu, \beta, L$ such that
    \begin{eqnarray}
        &&\mathbb{E}[\Delta(\pi_{\theta_{n}})]\nonumber\\
        &\leq &K\Delta(\pi_{\theta_0})\left(\frac{1}{n}\right)^{2\mu\beta} +\sum_{k=1}^n\frac{L\beta^2 C_k}{2k^2}\times K\left(\frac{k+1}{n}\right)^{2\mu\beta}\nonumber\\
        &\leq&K\Delta(\pi_{\theta_0})\left(\frac{1}{n}\right)^{2\mu\beta} +\sum_{k=1}^{n_0-1}\frac{L\beta^2 C_k}{2k^2}\times K\left(\frac{k+1}{n}\right)^{2\mu\beta}+ \sum_{k=n_0}^{n}\frac{L\beta^2 C_k}{2k^2}\times K\left(\frac{k+1}{n}\right)^{2\mu\beta}.\nonumber
    \end{eqnarray}
    Under Assumptions \ref{assump:boundedreward}, \ref{assump:kernelfun} and \ref{assump:score}, both $\widehat{g}_{\text{KAE}}(\theta)$ and $g(\theta)$ are uniformly bounded over $\theta\in\Theta$. Thus, the uniform boundedness applies to $\text{MSE}(\widehat{g}_{\text{KAE}}(\theta_n))$ as well, so that $C_n \leq M<\infty$ holds for all $n\in \mathbb{N}$. Setting $c_1 = K\Delta(\pi_{\theta_0})$, $c_2 = LK\beta^2M/2$ and $c_3=3^{2\mu\beta-1}LK\beta^2/(2\mu\beta-1)$, $\mathbb{E}[\Delta(\pi_{\theta_{n}})]$ can be further bounded by
    \begin{eqnarray}
        && c_1\left(\frac{1}{n}\right)^{2\mu\beta} + LK\beta^2M\sum_{k=1}^{n_0-1}\frac{1}{2k^2}\left(\frac{k+1}{n}\right)^{2\mu\beta} + LK\beta^2\max_{n_0<k\leq n}C_k\times n^{-2\mu\beta}\sum_{k=n_0}^n(k+1)^{2\mu\beta-2}\nonumber\\
        &\leq& c_1\left(\frac{1}{n}\right)^{2\mu\beta} + c_2n_0\left(\frac{n_0}{n}\right)^{2\mu\beta} + LK\beta^2\max_{n_0<k\leq n}C_k\times n^{-2\mu\beta}\sum_{k=n_0}^n(k+1)^{2\mu\beta-2}\nonumber\\
        &\leq& c_1\left(\frac{1}{n}\right)^{2\mu\beta} + c_2n_0\left(\frac{n_0}{n}\right)^{2\mu\beta} + c_3\frac{\max_{n_0<k\leq n}C_k}{n},\nonumber
    \end{eqnarray}
    where the last inequality follows from the following inequality
    \begin{eqnarray}
        \sum_{k=n_0}^n(k+1)^{2\mu\beta-2}\leq\int_{n_0+1}^{n+2} x^{2\mu\beta-2} \,dx<\frac{(n+2)^{2\mu\beta-1}}{2\mu\beta-1}<\frac{(3n)^{2\mu\beta-1}}{2\mu\beta-1}.
    \end{eqnarray}
    This finishes the proof of Theorem \ref{thm:policy}.

\subsection{Proof of Corollary \ref{coro:policy}}
By Theorem \ref{thm:policy}, for any $1\le n_0\le n$ and any
$0<\varepsilon<2\mu\beta-1$,
\[
    \mathbb E[\Delta(\pi_{\theta_n})]
    \le
    c_1 n^{-1-\varepsilon}
    +
    c_2 n_0
    \left(
    \frac{n_0}{n}
    \right)^{1+\varepsilon}
    +
    \frac{c_3}{n}
    \sup_{k\ge n_0}
    \mathrm{MSE}
    \left(
    \widehat g_{\mathrm{KAE}}(\theta_k)
    \right).
\]
Therefore, the only algorithm-dependent term in the upper bound is
\[
    \frac{c_3}{n}
    \sup_{k\ge n_0}
    \mathrm{MSE}
    \left(
    \widehat g_{\mathrm{KAE}}(\theta_k)
    \right).
\]

By Theorem \ref{thm:grad}, %with $\tau=p\wedge s$,
\[
    \mathrm{MSE}\left(\widehat g_{\mathrm{KAE}}(\theta_k)\right)
    =\mathrm{MSE}\left(\widehat g_{\mathrm{oracle}}(\theta_k)\right)
    + O\left(\frac{h_k^{2}}{BG}\right)
    +O\left(\frac{1}{k^2 h_k^2 BG}\right)+ O\left(\frac{m}{k h_k B^2G^2}\right)+O\left(\frac{m^2}{k^2B^3G^2}\right).
\]
If $h_k\to0$ and $k h_k\to\infty$, then
\[
    O\left(\frac{h_k^{2}}{BG}\right)
    +O\left(\frac{1}{k^2 h_k^2 BG}\right)+ O\left(\frac{m}{k h_k B^2G^2}\right)+O\left(\frac{m^2}{k^2B^3G^2}\right)=o(B^{-1})
\]
as $k\to\infty$. Hence, for any $n_0$ such that $n_0\to\infty$ as
$n\to\infty$,
\[
    \sup_{k\ge n_0}\mathrm{MSE}\left(\widehat g_{\mathrm{KAE}}(\theta_k)\right)
    = \sup_{k\ge n_0}\mathrm{MSE}\left(\widehat g_{\mathrm{oracle}}(\theta_k)\right)+o(B^{-1}).
\]
Substituting this relation into the policy suboptimality bound shows that the
upper bound achieved by KAE is asymptotically equivalent to the upper bound that
would be achieved by the oracle gradient estimator.

It remains to compare KAE with GRPO and REINFORCE++. By Corollary
\ref{coro:grad}, KAE's gradient estimator is asymptotically equivalent to the
oracle gradient estimator. In contrast, the
gradient estimators used by GRPO and REINFORCE++ have generally no smaller MSEs, by Corollary
\ref{coro:grad}. Since the suboptimality upper bound in Theorem \ref{thm:policy}
depends monotonically on the gradient-estimator MSE, KAE achieves an
asymptotic suboptimality upper bound no larger than those achieved by GRPO and
REINFORCE++.

This completes the proof.

\section{Lemmas}
\begin{lemma}\label{lem:Uniform-distribution}
    Let $I_1, \ldots, I_t$ be a sequence of independent Bernoulli random variables with common success probability $p$, and $N=\sum_{j=0}^{t-1} I_j$
    denote the total number of successes by time $t$. Conditional on $N=n>0$, let $S$ be the set of success times and let $X_1,\ldots,X_n$ be labeled by a uniformly random permutation of this set without order. Then, for each $i=0,\ldots,t-1$, $X_i$ follows a uniform distribution over $\{0,1,\ldots,t-1\}$. More specifically, for any $k\in\{1,\ldots,t\}$, %$k\neq l$,
    \begin{align*}
        &\mathbb{P}(X_i=k\mid N=n)=\frac{1}{t}.
        %&\mathbb{P}(X_i=k, X_j=l\mid N=n)=\frac{1}{t(t-1)}.
    \end{align*}
\end{lemma}

\begin{proof}[Proof of Lemma \ref{lem:Uniform-distribution}]
    For any subset $A\subseteq \{1,\ldots,t\}$ with $|A|=n$, we have
\begin{equation*}
    \mathbb{P}\left(\{j:I_j=1\}=A\right)=p^n(1-p)^{t-n}.
\end{equation*}
Hence, conditional on $N=n$, 
\begin{equation*}
    \mathbb{P}\bigl(\{j:I_j=1\}=A\mid N=n\bigr)=\frac{p^n(1-p)^{t-n}}{\mathbb{P}(N = n)}=\frac{1}{\binom{t}{n}}.
\end{equation*}
On the other hand, Fix $k\in\{1,\ldots,t\}$, the number of subsets of size $n$ containing $k$ is $\binom{t-1}{n-1}$.
Therefore,
\begin{equation*}
    \mathbb{P}( I_k = 1\mid N=n)=\mathbb{P}\left( k \in \{j: I_j = 1\}\mid N=n\right)
=\frac{\binom{t-1}{n-1}}{\binom{t}{n}}
=
\frac{n}{t}.
\end{equation*}

Since $X_i$ is one of the $n$ success times chosen uniformly among the success times, symmetry gives
\begin{equation*}
    \mathbb{P}(X_i=k\mid N=n)=\frac{1}{n}\cdot \frac{n}{t}=\frac{1}{t}.
\end{equation*}
This finishes the proof.
\end{proof}

\begin{lemma}\label{lem:Riemann-sum}
Let $f:[0,1]\to\mathbb{R}$ be Lipschitz continuous with Lipschitz constant $L>0$. Then, for every integer $n\geq 1$,
\begin{equation*}
    \left|\frac{1}{n}\sum_{i=1}^{n} f\left(\frac{i}{n}\right)-\int_0^1 f(t)\,dt\right|\leq \frac{L}{2n}.
\end{equation*}
\end{lemma}

\begin{proof}
We decompose the integral over the partition $i/n,i\in\{1,\ldots,n\}$ and obtain
\[
\frac{1}{n}\sum_{i=1}^{n} f\left(\frac{i}{n}\right)
-
\int_0^1 f(t)\,dt
=
\sum_{i=1}^{n}
\int_{\frac{i-1}{n}}^{\frac{i}{n}}
\left[
f\left(\frac{i}{n}\right)-f(t)
\right]dt.
\]
Taking absolute values on both sides and using the Lipschitz continuity of $f$, we obtain
\begin{eqnarray}
    \left|\frac{1}{n}\sum_{i=1}^{n} f\left(\frac{i}{n}\right)-\int_0^1 f(t)\,dt\right| 
    &\leq&\sum_{i=1}^{n}\int_{\frac{i-1}{n}}^{\frac{i}{n}}\left|f\left(\frac{i}{n}\right)-f(t)\right|dt\nonumber\\
    &\leq&\sum_{i=1}^{n}\int_{\frac{i-1}{n}}^{\frac{i}{n}}L\left|\frac{i}{n}-t\right|dt. \nonumber\\
    &=&\sum_{i=1}^{n}L\cdot\frac{1}{2n^2} = \frac{L}{2n}.\nonumber
\end{eqnarray}
This finishes the proof.
\end{proof}

\begin{lemma}\label{lem:binomial-bound}
Let $Y\sim \mathrm{Bin}(n,p), p\in(0,1]$ and let $a>0$ be a fixed constant. Then there exist constants $c_1,c_2>0$, depending only on $a$, such that
\begin{equation*}
    \frac{c_1}{np}\le\mathbb{E}\left[\frac{1}{Y+a}\right]\le\frac{c_2}{np}
\end{equation*}
holds for all $n\ge ap^{-1}$. 
\end{lemma}

\begin{proof}[Proof of Lemma \ref{lem:binomial-bound}]
We first prove the lower bound. Since the function $x\mapsto 1/x$ is convex on
$(0,\infty)$, Jensen's inequality gives
\begin{equation*}
    \mathbb{E}\left[\frac{1}{Y+a}\right]\ge\frac{1}{\mathbb{E}[Y]+a}=\frac{1}{np+a}.
\end{equation*}
Notice that for all $n>ap^{-1}$, we have $np+a\le 2np$ and thus
\begin{equation*}
    \mathbb{E}\left[\frac{1}{Y+a}\right]\ge\frac{1}{2np}.
\end{equation*}
We next prove the upper bound. For every $y\ge 0$,
\[
\frac{1}{y+a}
\le
\max\left\{1,\frac1a\right\}\frac{1}{y+1}.
\]
Therefore,
\[
\mathbb{E}\left[\frac{1}{Y+a}\right]
\le
\max\left\{1,\frac1a\right\}
\mathbb{E}\left[\frac{1}{Y+1}\right].
\]
It remains to bound $\mathbb{E}[1/(Y+1)]$. By direct calculation,
\begin{align*}
\mathbb{E}\left[\frac{1}{Y+1}\right]
&=\sum_{k=0}^n\frac{1}{k+1}\binom{n}{k}p^k(1-p)^{n-k}  \\
&=\frac{1}{n+1}\sum_{k=0}^n\binom{n+1}{k+1}p^k(1-p)^{n-k}  \\
&=\frac{1}{(n+1)p}\sum_{k=0}^n\binom{n+1}{k+1}p^{k+1}(1-p)^{n-k}  \\
&=\frac{1}{(n+1)p}\sum_{j=1}^{n+1}\binom{n+1}{j}p^j(1-p)^{n+1-j}  \\
&=\frac{1-(1-p)^{n+1}}{(n+1)p}\\
&\leq \frac{1}{(n+1)p}\leq \frac{1}{np}.
\end{align*}
Thus,
\[
\mathbb{E}\left[\frac{1}{Y+a}\right]
\le
\frac{\max\{1,1/a\}}{np}.
\]
Combining the lower and upper bounds gives
\[
\frac{1/2}{np}
\le
\mathbb{E}\left[\frac{1}{Y+a}\right]
\le
\frac{\max\{1,1/a\}}{np}.
\]
This completes the proof.
\end{proof}

\begin{lemma}\label{lem:uniform-product-bound}
Let $A>1$ and $B>0$. Suppose $n_0\in\mathbb{N}$ satisfies $n_0\ge 1$. For $n\ge n_0$, define
$$
a_n=\prod_{k=n_0}^{n}\left(1-\frac{A}{k}+\frac{B}{k^2}\right).
$$
Then there exists some constant $C$ only depends on $A$ and $B$ such that
\begin{equation*}
     a_n\le C\left(\frac{n_0}{n}\right)^A
\end{equation*}
holds for all $n \geq 1$.
\end{lemma}

\begin{proof}[Proof of Lemma \ref{lem:uniform-product-bound}]
We first consider the case when $n_0\ge A$. Then for every $k\ge n_0$,
$$
    1-\frac{A}{k}+\frac{B}{k^2}>0.
$$
Using the elementary inequality $\log(1+x)\le x, \forall x>-1$, we obtain
\begin{align*}
    \log a_n &=\sum_{k=n_0}^{n}\log\left(1-\frac{A}{k}+\frac{B}{k^2}\right) \\
    &\le\sum_{k=n_0}^{n}\left(-\frac{A}{k}+\frac{B}{k^2}\right) \\
    &=-A\sum_{k=n_0}^{n}\frac{1}{k}+B\sum_{k=n_0}^{n}\frac{1}{k^2}.
\end{align*}
For the harmonic sum, we have
\[
    \sum_{k=n_0}^{n}\frac{1}{k}
    \ge
    \int_{n_0}^{n+1}\frac{1}{x}\,dx
    =
    \log\left(\frac{n+1}{n_0}\right).
\]
Moreover,
\begin{equation*}
    \sum_{k=n_0}^{n}\frac{1}{k^2}\le\sum_{k=1}^{\infty}\frac{1}{k^2} =\frac{\pi^2}{6}<3.
\end{equation*}
Therefore,
\begin{align}
    \log a_n & \le-A\log\left(\frac{n+1}{n_0}\right) + 3B. \nonumber
\end{align}
Take exponential on both sides yields
\begin{equation*}
    a_n\leq \exp(3B)\left(\frac{n_0}{n+1}\right)^A <\exp(3B)\left(\frac{n_0}{n}\right)^A.
\end{equation*}
Next we consider the case when $1\leq n_0 < A$. In this case, $n_0A >A$. Let $M = \max_{k<A}(1-A/k +B/k^2)$, which only depends on $A$ and $B$. It follows that
\begin{eqnarray}
    a_n &=& \prod_{k=n_0}^{n}\left(1-\frac{A}{k}+\frac{B}{k^2}\right)\nonumber\\
    &\leq& [\max\{M,1\}]^{\lceil A\rceil}\times  \prod_{k=\lceil A\rceil+1}^{n}\left(1-\frac{A}{k}+\frac{B}{k^2}\right)\nonumber\\
    &\leq& [\max\{M,1\}]^{\lceil A\rceil}\times\exp(3B)\times\left(\frac{\lceil A\rceil+1}{n}\right)^{A}\nonumber\\
    &\leq& [\max\{M,1\}]^{\lceil A\rceil}\times(\lceil A\rceil+1)^A\times\exp(3B)\times\left(\frac{n_0}{n}\right)^A.\nonumber
\end{eqnarray}
Here the second-to-last inequality follows from the case when $n_0>A$ and the last inequality follows from $n_0\geq 1$. Hence, the proof is completed by taking 
\begin{equation*}
    C= \exp(3B)\times  [\max\{M,1\}]^{\lceil A\rceil}\times(\lceil A\rceil+1)^A.
\end{equation*}
\end{proof}

\begin{lemma}[First-order smooth interpolation]\label{lem:first-order-interpolation}
Let $\{\theta_i\}_{i=0}^n \subset \mathbb{R}^d$ be a sequence satisfying
\[
    \|\theta_{i+1}-\theta_i\| \leq \frac{c}{i},
    \qquad i=1,\ldots,n-1,
\]
for some constant $c>0$. Then there exists a continuously differentiable path $\theta(t),t\in[1,n]$ such that $\theta(i)=\theta_i$ for all $i=0,\ldots,n$ and $\|\theta'(t)\| \leq \frac{3c}{t}, \forall t \in [1,n]$. 
\end{lemma}

\begin{proof}
Let $\Delta_i=\theta_{i+1}-\theta_i, \quad i=0,\ldots,n-1$ and define $\rho(u)=3u^2-2u^3$. Then it is easy to verify that $\rho(0) = \rho'(0)=\rho'(1) =0$ and $\rho(1)=1$. For each $t\in[i,i+1]$, define
\[
    \theta(t)
    =
    \theta_i+\rho(t-i)\Delta_i,
    \qquad i=1,\ldots,n-1.
\]
Clearly, $\theta(i)=\theta_i$ for all $i=0,\ldots,n$.
Moreover, since $\rho'(0)=\rho'(1)=0$, the left and right derivatives agree at each integer point. Hence $\theta$ is continuously differentiable on $[1,n]$. Additionally, for $t\in[i,i+1]$, we have $\theta'(t)=\rho'(t-i)\Delta_i$. Since $\sup_{u\in[0,1]}|\rho'(u)| = \sup_{u\in[0,1]}6u(1-u) = 3/2$, it follows that 
$$
\|\theta'(t)\|\leq\frac{3}{2}\|\Delta_i\|\leq\frac{3c}{2i}.
$$
Since $t\in[i,i+1]$, we have $t\leq i+1\leq 2i$, and hence $i^{-1}\leq 2t^{-1}$. Thus,
$$
 \|\theta'(t)\|\leq\frac{3c}{2i}\leq\frac{3c}{t}.
$$
This proves the desired bound.
\end{proof}

\section{Experimental Details}
\subsection{Baseline construction in the one-shot experiment} \label{sup:one-shot-detail}

In the one-shot regime, the training set consists of a single prompt $x$. At training step $i$, the minibatch contains $B$ repeated copies of this prompt, so that $X_i^{(b)} = x$ for all $b=1,\ldots,B$. For each copy we sample $G$ completions and compute rewards
\[
Z_i^{(b,g)} = r\!\left(X_i^{(b)}, Y_i^{(b,g)}\right), \qquad b=1,\ldots,B,\quad g=1,\ldots,G.
\]
Because all $B$ repeated copies correspond to the same prompt, one may in principle construct a baseline from the entire current minibatch,
\[
\widetilde V_i^{\mathrm{pool},(b,g)}
=
\frac{1}{BG-1}\sum_{(b',g')\neq (b,g)} Z_i^{(b',g')},
\]
which uses all but the current reward among the $BG$ rewards collected at step $i$. Our oracle method in Figure~\ref{fig:oneshot} is exactly this baseline, and its advantage is given by
\[
A_i^{\mathrm{oracle},(b,g)}
=
Z_i^{(b,g)}-\widetilde V_i^{\mathrm{pool},(b,g)}.
\]

By contrast, the one-shot GRPO baseline used in \citet{wang2025reinforcement}\footnote{In the \texttt{verl} version used for one-shot RLVR, completions are grouped by a randomly generated \texttt{uid} attached to each repeated copy in the current minibatch. Consequently, the GRPO baseline for copy $b$ is computed only from the $G$ completions generated from that copy at the current step, rather than from all $BG$ completions in the minibatch.} uses only the rewards from the current local group. Its baseline is the within-group average
\[
\widehat V_i^{\mathrm{GRPO},(b,g)}
=
\frac{1}{G-1}\sum_{k\neq g} Z_i^{(b,k)},
\]
and the corresponding advantage is
\[
A_i^{\mathrm{GRPO},(b,g)}
=
Z_i^{(b,g)}-\widehat V_i^{\mathrm{GRPO},(b,g)}.
\]
Thus, GRPO estimates the baseline from only $G$ rewards at the current step, even though the one-shot setting permits pooling over the entire minibatch.

To make the effect of KAE visible in the one-shot experiment, we also treat the $B$ repeated copies in the minibatch as $B$ distinct prompt streams when constructing KAE. KAE therefore does not pool all repeated copies at the current step. Instead, for each repeated stream $b$, it first computes the current-step local mean
\[
\bar Z_{i,-g}^{(b)}
=
\frac{1}{G-1}\sum_{k\neq g} Z_i^{(b,k)},
\]
and then smooths these local means across nearby training steps. Writing $\bar Z_j^{(b)}=G^{-1}\sum_{k=1}^G Z_j^{(b,k)}$ for the historical group mean at step $j<i$, the resulting baseline can be expressed as
\[
\widehat V_i^{\mathrm{KAE},(b,g)}
=
\frac{\sum_{j<i} K_h(i-j)\,\bar Z_j^{(b)} + K_h(0)\,\bar Z_{i,-g}^{(b)}}{\sum_{j<i} K_h(i-j) + K_h(0)}.
\]
The resulting advantage is
\[
A_i^{\mathrm{KAE},(b,g)}
=
Z_i^{(b,g)}-\widehat V_i^{\mathrm{KAE},(b,g)}.
\]
Therefore, in the one-shot comparison, the oracle baseline borrows information across all repeated copies at the current step, GRPO borrows information only within the current local group, and KAE borrows information across nearby training steps for each repeated stream.

\subsection{Value-estimation MSE diagnostic}
\label{sup:mse-detail}

We next describe the empirical MSE diagnostic used in Section~\ref{sec:exp-baseline-mse}. Fix a target training step $i$ and a prompt $x_b$ from the sticky minibatch used throughout training. In the implementation, we generate a separate target-step sample pool
\[
\left\{Z_{i,m}^{\mathrm{oracle}}(x_b)\right\}_{m=1}^{N_{\mathrm{oracle}}}
\]
to approximate the target value function by
\[
\widetilde V_i(x_b)
=
\frac{1}{N_{\mathrm{oracle}}}\sum_{m=1}^{N_{\mathrm{oracle}}}Z_{i,m}^{\mathrm{oracle}}(x_b).
\]
This quantity is treated as the oracle target in the MSE calculation.

For each Monte Carlo repeat $r=1,\ldots,R$, we draw a small current-step group from a disjoint target-step estimation pool and, for KAE, additional small groups from the history windows $\mathcal H_i$. We then compute the corresponding leave-one-out baselines for GRPO and KAE using the same constructions as in Section~\ref{sec:alg}, while REINFORCE++ uses the repeat-specific across-prompt average of the target-step group means. Denoting by $\widehat V_{i,r,-g}^{(a,b)}$ the resulting baseline assigned to the $g$th sampled response of prompt $x_b$ under method $a\in\{\mathrm{Naive},\mathrm{RF++},\mathrm{GRPO},\mathrm{KAE}\}$, the prompt-level empirical MSE is defined by
\[
\widehat{\mathrm{MSE}}_b^{(a)}
=
\frac{1}{RG}\sum_{r=1}^{R}\sum_{g=1}^{G}
\left(
\widehat V_{i,r,-g}^{(a,b)}-\widetilde V_i(x_b)
\right)^2.
\]
The quantities reported in Table~\ref{tab:baseline_mse_cross} are obtained by averaging $\widehat{\mathrm{MSE}}_b^{(a)}$ over prompts $b=1,\ldots,B$. Confidence intervals in the sensitivity plots are obtained by bootstrap resampling over prompts.

\subsection{Gradient-evaluation diagnostic}
\label{sup:grad-detail}

We also describe the empirical gradient diagnostic used in Section~\ref{sec:exp-gradvar}. Fix a target training step $i$ and prompt $x_b$. Let
\[
\left\{(Y_{i,h}^{(b)},Z_{i,h}^{(b)})\right\}_{h=1}^{N_{\mathrm{ref}}}
\]
denote the full target-step sample pool retained for this prompt. We use it to define a shared reference value
\[
\widetilde V_i^{(b)}
=
\frac{1}{N_{\mathrm{ref}}}\sum_{h=1}^{N_{\mathrm{ref}}}Z_{i,h}^{(b)}
\]
and the corresponding reference gradient
\[
\widetilde g_i^{(b)}
=
\nabla_{\theta}
\left\{
-\frac{1}{N_{\mathrm{ref}}}\sum_{h=1}^{N_{\mathrm{ref}}}
\left(Z_{i,h}^{(b)}-\widetilde V_i^{(b)}\right)
\log \pi_{\theta_i}\!\left(Y_{i,h}^{(b)}\mid x_b\right)
\right\}.
\]
This shared reference gradient is used as the center for all compared methods.

For each Monte Carlo repeat $m=1,\ldots,M$, we sample with replacement $G$ responses from every step in $\mathcal H_i\cup\{i\}$ and then construct the corresponding no-standardization advantages for Naive, REINFORCE++, GRPO, and KAE using the same baseline definitions as in Section~\ref{sec:alg}. Let $A_{i,m}^{(a,b,g)}$ denote the resulting advantage assigned to the $g$th sampled response of prompt $x_b$ under method $a$. We then form the stochastic loss
\[
L_{i,m}^{(a,b)}
=
-\frac{1}{G}\sum_{g=1}^{G}
A_{i,m}^{(a,b,g)}
\log \pi_{\theta_i}\!\left(Y_{i,m}^{(b,g)}\mid x_b\right),
\]
and the corresponding gradient sample
\[
\widehat g_{i,m}^{(a,b)}
=
\nabla_{\theta}L_{i,m}^{(a,b)}.
\]
The prompt-level empirical error reported by the diagnostic is
\[
\widehat{\mathcal E}_b^{(a)}
=
\frac{1}{M}\sum_{m=1}^{M}
\left\|
\widehat g_{i,m}^{(a,b)}-\widetilde g_i^{(b)}
\right\|_2^2,
\]
Table~\ref{tab:gradvar_stage_summary} reports the average of $\widehat{\mathcal E}_b^{(a)}$ across prompts. For the 7B MATH model, this quantity is evaluated on the tracked parameter subset used in the implementation, namely the final decoder block and the output layer.

\subsection{Policy-optimization configurations}
\label{sup:policy-config}

We finally summarize the main training configurations used in the policy-optimization experiments of Section~\ref{sec:large-scale}. Across all settings, we use an actor learning rate of $10^{-6}$, do not add KL regularization to the reward or the policy loss, set the entropy coefficient to zero, and enable gradient checkpointing. Within each experimental block, the compared methods share the same model, data split, rollout budget, and optimization setup; they differ only in the way the baseline is constructed and, when applicable, in whether KAE-style reuse of recent prompt occurrences is enabled.

\smallskip

\noindent\textbf{7B multi-stream experiments.}
For the 7B MATH experiment, we use Qwen2.5-Math-7B, train on the harder MATH training split (Level 3-5). Prompts are truncated at 1024 tokens and responses at 2048 tokens. Each update uses a minibatch of 256 prompts, a PPO minibatch size of 128, and $G=8$ completions per prompt. When KAE is enabled, each prompt can be reused for up to ten nearby updates, and the temporal weighting follows the triangular kernel in Section~\ref{sec:alg} with bandwidth $h\approx 10$. Training is run on 4 GPUs for 200 update steps, with evaluation every 5 steps.

For the 7B DAPO experiment, we again use Qwen2.5-Math-7B, now trained on the DAPO-MATH 17k dataset. The prompt budget remains 1024 tokens, and the response budget is 2048 tokens. The minibatch size is 256, the PPO minibatch size is 128, and the rollout group size is $G=4$. Here KAE reuses each prompt for up to eight nearby updates, and uses the same triangular kernel with bandwidth $h\approx 10$. Training is run on 4 GPUs for 500 update steps, with evaluation every 10 steps.

\smallskip

\noindent\textbf{1.5B single-stream experiments.}
For the 1.5B GSM8K experiment, we use Qwen2.5-1.5B-Instruct, train on the GSM8K training set, and evaluate on the GSM8K test set. This is the single-stream regime, so each prompt generates only one completion per update. To keep the total rollout budget comparable, we therefore use a much larger minibatch of 1024 prompts and a PPO minibatch size of 512. Both prompts and responses are capped at 512 tokens. The run lasts 400 update steps, evaluation is performed every 5 steps, and the reported results are averaged over repeated runs. In the KAE experiments, a prompt may be revisited for up to ten nearby updates, and the triangular kernel uses bandwidth $h\approx 10$ throughout training.

For the 1.5B MATH experiment, we use Qwen2.5-Math-1.5B, again in the single-stream regime with one completion per prompt and the same total rollout budget of 1024 prompt-response pairs per update. The minibatch size is therefore 1024 and the PPO minibatch size is 512. Prompts are truncated at 1024 tokens and responses at 2048 tokens. Training lasts 400 update steps with evaluation every 5 steps. As in the 1.5B GSM8K setting, KAE revisits each prompt for up to ten nearby updates and uses the same triangular kernel with bandwidth $h\approx 10$ throughout training.

\section{Supplementary Figures}
\begin{figure}[htbp]
  \centering
  \includegraphics[width=\linewidth]{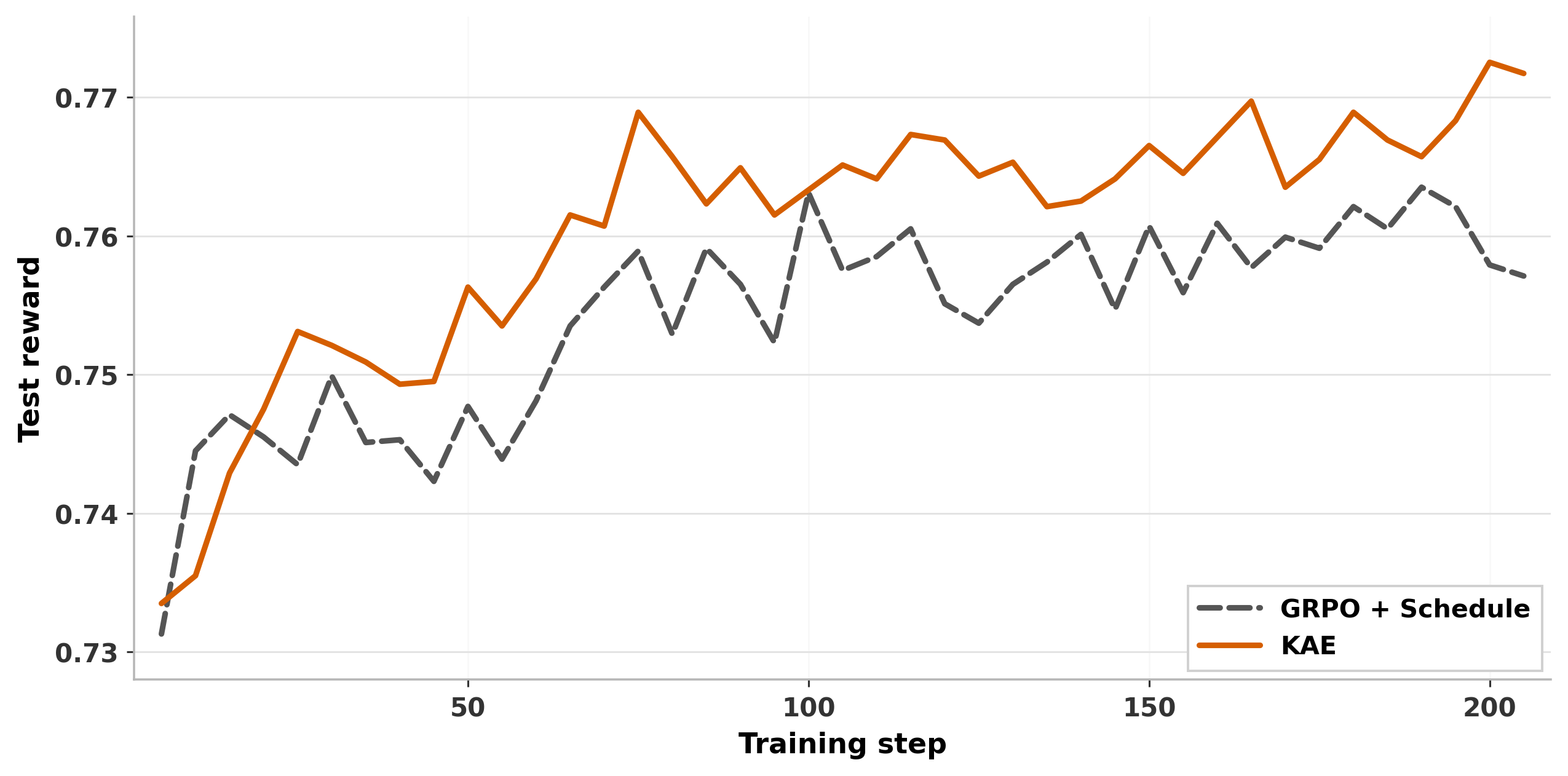}
  \caption{Test accuracy of Qwen2.5-Math-7B models trained with KAE (orange) and a GRPO variant using the proposed prompt sampling scheme (grey), on the MATH dataset.}\label{fig:ablation_math_curves}
\end{figure}

\end{document}